\documentclass{article}

\PassOptionsToPackage{numbers, sort, compress}{natbib}
 \usepackage[preprint]{neurips_2026}

\usepackage[utf8]{inputenc} 
\usepackage[T1]{fontenc}    
\usepackage{hyperref}       
\usepackage{url}            
\usepackage{booktabs}       
\usepackage{amsfonts}       
\usepackage{nicefrac}       
\usepackage{microtype}      
\usepackage[table]{xcolor}  
\usepackage{amsmath}
\usepackage{amssymb}
\usepackage{mathtools}
\usepackage{amsthm}
\usepackage[capitalize,noabbrev]{cleveref}
\usepackage{graphicx}
\usepackage{subfigure}
\usepackage{multirow}
\usepackage{xspace}
\usepackage{enumitem}
\usepackage{wrapfig}
\usepackage[most]{tcolorbox}

\definecolor{recipeblue}{HTML}{FFECBD}
\definecolor{recipebluebg}{HTML}{FFFBEB}

\newtcolorbox{recipebox}[1]{%
  enhanced,
  colback=white,
  colframe=recipeblue,
  coltitle=black,
  colback=recipebluebg,
  colbacktitle=recipeblue,
  fonttitle=\bfseries,
  title={#1},
  boxrule=1.6pt,
  arc=2pt,
  left=6pt,
  right=6pt,
  top=6pt,
  bottom=6pt,
  titlerule=0pt,
  toptitle=2pt,
  bottomtitle=2pt,
}

\usepackage[textsize=tiny]{todonotes}

\usepackage{booktabs}
\usepackage{multirow}

\usepackage{multirow}
\usepackage[table]{xcolor}
\definecolor{c1}{RGB}{255,193,193}
\definecolor{c2}{RGB}{255,224,178}
\definecolor{c3}{RGB}{255,249,196}

\newtheorem{theorem}{Theorem}[section]
\newtheorem{proposition}[theorem]{Proposition}

\newtheorem{definition}[theorem]{Definition}

\newtheorem{remark}[theorem]{Remark}

\newcommand{\modelname}{\textsc{Mochi}} 

\title{\modelname{}: Aligning Pre-training and Inference for Efficient Graph Foundation Models via Meta-Learning}

\author{%
  Jo\~ao Mattos$^*$, Arlei Silva$^*\ddagger$ \\
  $^*$Computer Science Department; $^\ddagger$Ken Kennedy Institute\\
  Rice University\\
  Houston, TX, USA \\
  \texttt{\{jrm28, arlei\}@rice.edu}
}

\begin{document}

\maketitle

\begin{abstract}

We propose \modelname{}\footnote{Source code available at: \url{https://github.com/joaopedromattos/mochi/}}, a Graph Foundation Model (GFM) that aligns pre-training with the few-shot inference protocol used at evaluation. Existing GFMs typically follow a pretrain-then-unify pattern: an encoder is trained with self-supervised, reconstruction, contrastive, or prompt-based objectives, while downstream tasks are unified through a prototype-based readout. However, existing pipelines instantiate prototypes in a particularly restricted linear form that interacts poorly with their pre-training objectives, exposing three structural limitations: origin anchoring, convex-hull inclusion, and norm-coupled miscalibration. \modelname{} addresses this mismatch with episodic meta-learning: each step samples a few-shot node-, link-, or graph-level task and fits a closed-form ridge readout on the support set, propagating gradients through the solve to update the encoder. Since the meta-objective directly matches inference, \modelname{} avoids the costly full-graph negative sampling required by reconstruction- and contrastive-style pre-training. Across 20 evaluation tasks spanning node, link, and graph classification, \modelname{} and its per-step task-balanced variant \modelname{}++ achieve competitive aggregate performance while requiring 8$\sim$27 times less training time than the strongest baseline.
\end{abstract}

\section{Introduction}

Universal Graph Foundation Models (GFMs) aim to learn transferable representations that generalize across domains, tasks, and structural abstraction levels, including nodes, edges, subgraphs, motifs, and entire graphs. A central design challenge is \emph{task unification}: a model trained once should support heterogeneous downstream tasks through a common inference interface, while adapting from only a small labeled support set.

Most current GFMs follow a pretrain-then-unify pattern: an encoder is trained with self-supervised, reconstruction, contrastive, prompt-based, or task-tree objectives, and downstream tasks are then mapped to a lightweight readout~\citep{xia2024anygraph,fang_universal_2024,wang2024gft,liu2023graphprompt}. A common choice is \emph{prototype-based classification}~\citep{xia2024anygraph,liu2023graphprompt,liu_one_2024,he2025unigraph,wang2024gft}: class prototypes are formed by averaging labeled support embeddings, and queries are classified by similarity to these prototypes. This mechanism is attractive because it is parameter-free, supports arbitrary label sets, and reuses one scoring rule across node-, edge-, and graph-level tasks.

However, existing GFM pipelines typically instantiate prototypes in a particularly restricted form. Many use unnormalized inner-product similarity, either for simplicity or to match the inner-product heads used by link-prediction or reconstruction-style pre-training objectives~\citep{xia2024anygraph,liu2023graphprompt,liu_one_2024,he2025unigraph,wang2024gft}. Under this scoring rule, prototype classification becomes a restricted linear classifier whose weights are fixed to support-set centroids and whose bias terms are absent~\citep{snell2017prototypical,hou2022closer}. We show that this design is vulnerable to three structural failure modes in GFM task unification: (i) \emph{origin anchoring}, where decision boundaries are locked to the origin (\S\ref{sec:origin-anchored-failure}); (ii) \emph{convex-hull inclusion}, where a class whose prototype lies in the convex hull of the others cannot be the unique winner (\S\ref{sec:convex-hull}); and (iii) \emph{norm-coupled miscalibration}, where softmax confidence is driven by embedding norms rather than calibrated posterior probabilities (\S\ref{sec:calibration}).

These readout limitations are compounded by a second mismatch: pre-training objectives in current GFMs are only indirectly aligned with the few-shot inference protocol used at evaluation. Reconstruction, link-prediction, and contrastive objectives shape embeddings to preserve structural or connectivity patterns through large numbers of negative samples~\citep{qiu2018network}, rather than support-set adaptation and query-set prediction. The result is pre-training that is both statistically misaligned with downstream tasks and computationally expensive.

We propose \modelname{}, a GFM trained via episodic meta-learning that addresses both mismatches. At each step, the encoder receives a few-shot episode sampled from a node-, link-, or graph-level task; a closed-form ridge readout is fit on the support set, and gradients from the query loss flow through the solve to update the encoder. This aligns pre-training directly with the evaluation protocol while preserving a task-unified inference interface. Since the ridge readout is a learned linear classifier rather than a fixed centroid rule, it is not structurally constrained in the ways identified above. Training cost is also substantially reduced: because the meta-objective is the downstream few-shot task itself, \modelname{} uses negatives only at the small support-query scale of link-prediction episodes, avoiding the full-graph negative sampling that drives the cost of reconstruction- and contrastive-style pre-training.

Across 20 evaluation tasks spanning node, link, and graph classification, \modelname{} and a per-step task-balanced variant \modelname{}++ achieve competitive aggregate performance, with particularly strong node-classification results and consistent performance across task families. \modelname{} and \modelname{}++ train in 0.4 and 1.28 GPU-hours respectively, 8$\sim$27$\times$ less than the strongest evaluated baseline.

Our main contributions are:
\begin{itemize}
    \item \textbf{Prototype limitations in GFM task unification.} We show that the restricted prototype readouts adopted by current GFM pipelines exhibit origin anchoring, convex-hull inclusion, and norm-coupled miscalibration, and that common graph pre-training objectives can amplify these issues through the embedding geometry they induce (Section~\ref{sec:lp-unify}).

    \item \textbf{Efficient episodic meta-learning for GFMs.} We introduce \modelname{}, a task-unified GFM that meta-trains an encoder through a differentiable closed-form ridge readout across node-, link-, and graph-level episodes, aligning pre-training with few-shot inference (Section~\ref{sec:ttt}).

    \item \textbf{Accuracy-cost evaluation.} Across 20 evaluation tasks, \modelname{} and \modelname{}++ achieve competitive aggregate performance with 8$\sim$27$\times$ lower training time than the strongest baseline; synthetic experiments isolate the proposed failure modes (Section~\ref{sec:experiments}).
\end{itemize}

  \begin{figure*}[t]
  \centering
\includegraphics[width=0.9\textwidth]{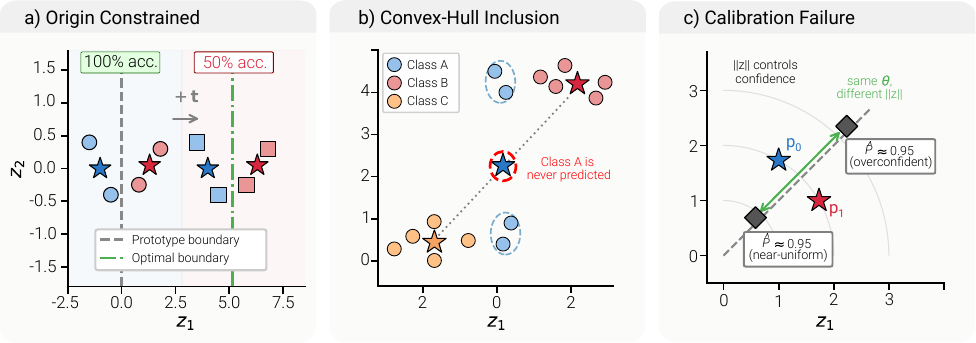}
  \caption{Three pitfalls of prototype classifiers in GFMs. $\bigstar$ denotes class prototypes, {\Large$\bullet$} original embeddings, {$\blacksquare$} translated embeddings, $\circ$ query nodes. \textbf{(a)~Origin-Anchored Failure} (\S\ref{sec:origin-anchored-failure}): a translation $\mathbf{t}$ moves the data cloud but the origin-anchored boundary (dashed) stays fixed, collapsing accuracy. \textbf{(b)~Convex-Hull Prototype Failure} (\S\ref{sec:convex-hull}): the prototype $\mathbf{p}_A$ of a bimodal class falls inside the convex hull of the other prototypes, so $A$ is never the argmax. \textbf{(c)~Calibration Failure} (\S\ref{sec:calibration}): query norms $\|\mathbf{z}\|$ control softmax confidence independently of posterior probability.}
  \label{fig:prototype-pitfalls}
  \end{figure*}

\section{GFM Task Unification and Pitfalls of Restricted Prototype Readouts}
\label{sec:lp-unify}

\paragraph{(Universal) Graph Foundation Models.} Let $G=(V,E)$ be an undirected graph with $n=|V|$ nodes, adjacency matrix $A\in\{0,1\}^{n\times n}$, and optional node features $X\in\mathbb{R}^{n\times d}$. A GFM can be decomposed into three stages~\citep{wang2025graph_2,wang2025graph}: (i)~\emph{feature and structural alignment}, which maps heterogeneous inputs to a shared space, typically via SVD over node features or the graph Laplacian~\citep{xia2024anygraph}; (ii)~\emph{pre-training}, in which a graph encoder $f_\theta$ mapping $(A,X)$ to node embeddings $\{z_v \in \mathbb{R}^{d_z}\}_{v\in V}$ is trained to yield transferable representations; and (iii)~\emph{task unification}, which combines these embeddings into a single inference interface across node-, edge-, subgraph-, and graph-level tasks. Our work targets (ii) and (iii): we identify pitfalls induced by prototype-based task unification, and propose a meta-learning framework that yields a task-general, sample-efficient training procedure.

\paragraph{Prototype classification setup.}
Given labeled supports $S_c$ for each class $c$, prototype classifiers~\citep{snell2017prototypical} compute class prototypes 
$\mathbf{p}_c = |S_c|^{-1}\sum_{v \in S_c}\mathbf{z}_v$
and classify queries by prototype similarity or distance. While squared-Euclidean distance yields a translation-invariant rule, many GFM pipelines use unnormalized inner-product scores, often for simplicity or to reuse link-prediction or reconstruction heads:
$\hat{y}_u = \arg\max_c \langle \mathbf{z}_u,\mathbf{p}_c\rangle$.
This gives a unified pairwise-scoring interface, but restricts the classifier to centroid-fixed weights and no bias terms~\citep{snell2017prototypical,hou2022closer}. We analyze the inner-product readout.

\subsection{Origin-Anchored Prototype Failure}
\label{sec:origin-anchored-failure}


Prototype classifiers with inner-product scores place pairwise decision boundaries through the origin: for classes $i$ and $j$, the boundary is
$\{z : \langle z, p_i-p_j\rangle = 0\}$.
Thus, predictions depend not only on the relative geometry of the classes, but also on the global position of the embedding cloud. If all embeddings are translated, pairwise distances, affine separability, and inter-centroid separations remain unchanged, yet the origin-anchored boundary does not move with the data. A task that remains linearly separable can therefore become unclassifiable by this readout.

\begin{remark}[Translation sensitivity of origin-anchored prototypes]
\label{rem:translation}
Let $\mathcal{Z}=\{z_1,\ldots,z_n\}\subset\mathbb{R}^{d_z}$ with labels $y_1,\ldots,y_n\in[C]$, and let $t\in\mathbb{R}^{d_z}$. Under the translation $z_i' = z_i + t$, prototype differences are invariant, $p_i' - p_j' = p_i - p_j$, so the prototype boundary remains anchored at the origin while the data cloud shifts by $t$. Its position relative to the data is then governed by the embedding mean rather than by task geometry, motivating centering remedies used in few-shot image classification~\citep{wang2019simpleshot,hou2022closer}.
\end{remark}

When the translation along $p_i - p_j$ exceeds the maximum within-class spread in that direction, all points fall on the same side of the boundary and a balanced binary task collapses to $50\%$ accuracy despite remaining trivially separable (Figure~\ref{fig:prototype-pitfalls}a). Prototype accuracy is thus governed by the position of the data relative to the origin rather than by representation quality. We verify this on synthetic data in Section~\ref{sec:experiments}; Appendix~\ref{ap:lp-geometry} shows that link-prediction pre-training in GFMs naturally produces embeddings whose mean lies far from the origin, triggering the failure.

\subsection{Convex-Hull Prototype Failure}
\label{sec:convex-hull}


A second limitation of prototype-based inference is its reliance on a single centroid per class. This compression is effective when classes are well represented by their means, but becomes brittle when a prototype falls in a geometrically unfavorable location. Under the origin-anchored score $s_c(z)=\langle z,p_c\rangle$, failure occurs whenever a class prototype lies inside the convex hull of the remaining prototypes (Figure~\ref{fig:prototype-pitfalls}b): it can never align with any query more strongly than all competing prototypes (see example in
Appendix~\ref{ap:multimodal-collapse-example}). Let $\mathrm{conv}(S)$ denote the convex hull of a set $S \subset \mathbb{R}^{d_z}$:

\begin{proposition}[Convex-hull inclusion prevents unique argmax]
\label{prop:domination}
If $p_c \in \mathrm{conv}\bigl(\{p_{c'} : c' \neq c\}\bigr)$, then for every embedding $z \in \mathbb{R}^{d_z}$,
$\langle z, p_c \rangle \leq \max_{c' \neq c} \langle z, p_{c'} \rangle.$
Consequently, class $c$ can never be the unique argmax of the prototype
scores.
\end{proposition}

The proof is deferred to Appendix~\ref{ap:proof-domination}. Proposition~\ref{prop:domination} identifies a sharp geometric failure mode:
once a class prototype lies inside the convex hull of the remaining
prototypes, the prototype rule cannot assign that class as the unique winner.
Thus, even if the underlying class remains recoverable, summarizing it by a
single centroid can render it effectively invisible to the prototype readout.
While multimodality is one way this can occur, the condition is more general:
any class geometry that pulls its centroid between competing prototypes can
trigger the same mismatch.

Exact convex-hull inclusion is a boundary case, but it captures a broader
instability. In Appendix~\ref{ap:eps-domination}, we introduce a soft
$\varepsilon$-inclusion variant showing that nearly included prototypes can
achieve at most a vanishingly small score advantage. As a result, their
predictions become highly unstable: small perturbations to the embeddings or
support set can eliminate any residual preference for the class.


\paragraph{Structural sources in graphs.}
Unfavorable centroid geometry can arise naturally in graph domains. Link-prediction pre-training often encodes structural properties such as degree or local connectivity through embedding norms and directions~\citep{qiu2018network,HuangSS21}. A semantic class may therefore mix structurally different nodes, such as hubs and peripheral nodes, whose centroid lies between subgroups rather than representing either one. When such centroids are pulled toward competing classes, they can approach or enter the convex hull of other prototypes. Thus, this failure is not restricted to multimodal examples: it can arise whenever one centroid poorly fits the geometry induced by the encoder.

\subsection{Norm-Coupled Miscalibration}
\label{sec:calibration}

Calibrated probabilities, e.g., for active sample selection under few-shot label budgets, are typically obtained by applying a softmax to prototype scores,
$P(\hat{y}=c \mid z)=\frac{\exp(\langle z, p_c\rangle)}{\sum_{c'=1}^C \exp(\langle z, p_{c'}\rangle)},$
but these scores are not naturally calibrated (Figure~\ref{fig:prototype-pitfalls}c). Two structural properties make this problematic.

\paragraph{Prototype norms affect logits independently of class prevalence.}
The logit for class $c$ is $\langle z, p_c\rangle = \|z\|\,\|p_c\|\cos\theta,$
and therefore depends on the prototype norm $\|p_c\|$. This norm is determined
by the geometry of the class embeddings (i.e. their location,
concentration, and multimodality) rather than by class prior probability or
any calibrated estimate of $P(y=c \mid z)$. As a result, the softmax can
systematically favor classes whose prototypes happen to have larger norms,
independently of their true prevalence or predictive uncertainty.

\paragraph{Query norms control confidence.}
The same decomposition shows that $\|z\|$ rescales all logits and therefore
controls the sharpness of the softmax distribution. Larger query norms produce
more confident predictions, while points near the origin yield softer, nearly
uniform outputs. Since link-prediction embeddings tie norm to degree (\S\ref{sec:convex-hull}), predictive confidence becomes entangled with graph structure rather than task-specific uncertainty. Simple image-domain remedies such as L2-normalizing support embeddings~\citep{wang2019simpleshot} are therefore unsuitable here, since they erase the very norm variation that encodes this structural signal. Standard post-hoc calibration methods~\citep{guo2017calibration,platt1999probabilistic} require held-out labeled data,
which is unavailable in the few-shot regime targeted by prototype inference.

\subsection{Implications for GFM Design}
\label{sec:implications}




The failure modes above show that the restricted inner-product prototype readouts used in current GFM pipelines are too rigid for robust task unification: they anchor boundaries at the origin, compress each class to one vector, and adapt to the support set only by recomputing centroids. This is limiting for GFMs, where one embedding space must support diverse node-, edge-, and graph-level tasks.

These observations suggest a simple design principle: GFM readouts should retain the few-shot appeal and task-unified interface of prototypes, while allowing lightweight task-specific adaptation at inference time. In Section~\ref{sec:ttt}, we instantiate this idea with meta-learning~\citep{finn2017model,AndrychowiczDCH16,jin_empowering_2023}, formulating inference as a convex optimization problem over the support set that learns discriminative weights and biases without introducing a task-specific decoder.

\begin{figure*}[htbp]
\centering
\includegraphics[width=0.87\textwidth]{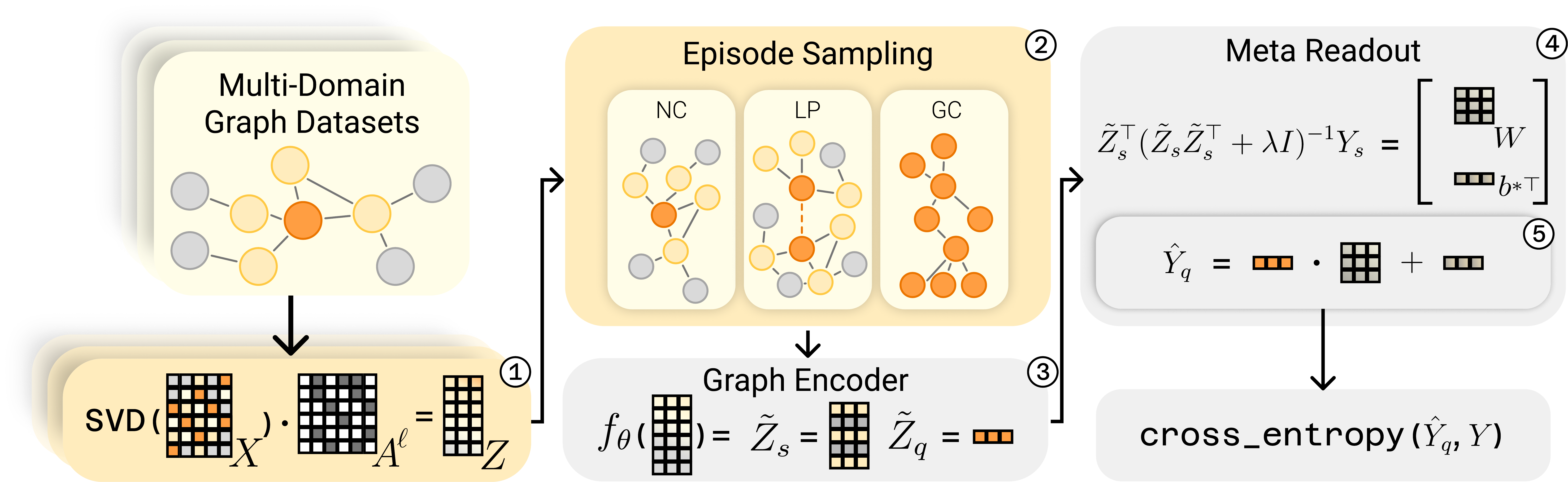}
\caption{\modelname{} preprocessing and training pipeline (\S\ref{sec:ttt}). (1)~Frozen SVD + $\ell$-hop adjacency propagation yields per-node input $Z$. (2)~Episodes are sampled across node, edge, and graph tasks (\S\ref{sec:meta_unification}). (3)~The GAMLP encoder produces support/query representations $\tilde{Z}_s,\tilde{Z}_q$. (4)~The closed-form meta-readout fits $W^*, b^*$ on $\tilde{Z}_s$ (Eq.~\ref{eq:ridge-dual}). (5)~It scores $\tilde{Z}_q$ (Eq.~\ref{eq:logits}), and cross-entropy trains the encoder.}
\label{fig:method-pipeline}
\end{figure*}

\section{Episodic Meta-Learning for Graph Foundation Models}
\label{sec:ttt}

Section~\ref{sec:lp-unify} showed that prototype inference fails because it is
too rigid: it fixes class weights to centroids, anchors decision boundaries at
the origin, and converts dot products into confidence scores with no mechanism
for task-specific adaptation. Replacing prototypes with a stronger linear
readout at inference time helps, but only partially~\citep{lee2019meta}. Ridge or logistic
regression can learn weights and biases from the support set, yet they still
operate on frozen embeddings produced by an objective that is
mismatched to few-shot classification or just operates over the assumption 
that downstream tasks align with the pre-training objectives~\citep{you2019position,srinivasan2019equivalence} (e.g. link prediction, contrastive learning, etc).

We therefore shift the adaptation target from the readout alone to the entire
training procedure. We propose \modelname{}, a Graph Foundation Model trained
via episodic meta-learning~\citep{finn2017model,bertinetto2018meta,lee2019meta} across node classification (NC), link prediction (LP), and
graph classification (GC) tasks (Figure~\ref{fig:method-pipeline}). In each episode, \modelname{} fits a closed-form
ridge classifier on a small support set and updates the encoder through query
loss. This preserves the minimal-data appeal of prototype-style inference while
aligning training with the evaluation protocol. It also yields an efficient
training pipeline: the episode head is solved in closed form, avoiding the
iterative inner-loop optimization required by gradient-based meta-learners~\citep{finn2017model,zhou2019metagnn,huang2020gmeta}. As we show in Section~\ref{sec:experiments}, this alignment
makes \modelname{} both competitive with strong GFM baselines and substantially
more training-efficient.

\subsection{Architecture}
\label{sec:meta_architecture}

\modelname{} consists of a graph encoder $f_\theta$ and a differentiable ridge
readout. The encoder follows a common feature-alignment pattern in graph GFMs~\citep{xia2024anygraph,sun_all_2023,yu_samgpt_2025}: for each training graph we precompute a frozen initial representation $X_0 \in \mathbb{R}^{|V| \times d_z}$ via SVD over node features and the symmetrically normalized adjacency (full construction in Appendix~\ref{ap:real-world-exp-setup}), propagate it along $\ell = 3$ hops of the adjacency, and concatenate the propagations into a per-node input stack $[A^k X_0]_{k=0}^{\ell}$. A trainable GAMLP~\citep{zhang2022graph} applies node-adaptive hop attention over this stack to produce node embeddings $z_v = f_\theta(A,X)_v \in \mathbb{R}^{d_z}$. Only GAMLP is meta-trained; the SVD and hop propagations are deterministic and do not receive gradients (Figure~\ref{fig:method-pipeline}).

Given a support set with embeddings $Z_s \in \mathbb{R}^{n_s \times d_z}$
(one row per episode example, constructed per task type as described in
\S\ref{sec:meta_unification}) and one-hot labels $Y_s \in \{0,1\}^{n_s \times C}$,
the readout solves the $\ell_2$-regularized least-squares problem in dual form:
\begin{equation}
\label{eq:ridge-dual}
\begin{bmatrix} W^* \\ b^{*\top} \end{bmatrix}
=
\tilde{Z}_s^\top
(\tilde{Z}_s \tilde{Z}_s^\top + \lambda I)^{-1}
Y_s,
\end{equation}
where $\tilde{Z}_s = [Z_s \mid \mathbf{1}] \in \mathbb{R}^{n_s \times (d_z+1)}$
augments the support embeddings with an all-ones bias column, $W^* \in \mathbb{R}^{d_z \times C}$
is the weight matrix, and $b^* \in \mathbb{R}^{C}$ is the bias vector. Query logits are then

\begin{equation}
\label{eq:logits}
\hat{Y}_q = \tilde{Z}_q \begin{bmatrix} W^* \\ b^{*\top} \end{bmatrix}.
\end{equation}

\paragraph{Why a ridge readout.}
Ridge regression provides a simple differentiable linear readout with a bias term. The bias frees decision boundaries from the origin, addressing origin anchoring (\S\ref{sec:origin-anchored-failure}), while learned class weights avoid the centroid-fixed geometry behind convex-hull inclusion (\S\ref{sec:convex-hull}); we verify both effects in Section~\ref{sec:exp-synthetic} and Appendix~\ref{ap:posthoc-linear}. However, applying ridge post hoc to existing GFMs is insufficient: it assumes their pre-training objectives already produce embeddings that are linearly separable under a few-shot readout. We therefore embed the ridge solver directly in the meta-training loop. By differentiating through the closed-form solve, \modelname{} optimizes the representation space for few-shot linear separability rather than relying on it as an inference-time assumption. Logistic regression offers similar flexibility but requires an iterative inner loop.

\subsection{Task-Unified Episodes}
\label{sec:meta_unification}

A central design principle of \modelname{} is that task unification happens at
the readout interface. Every episode is converted into the same input format for
the ridge head: a matrix of embeddings together with class labels. The only
task-specific component is how those embeddings are constructed from the encoder
output.

\begin{description}[leftmargin=*,itemsep=2pt]
\item[Node classification.] Each example is a node embedding
      $z_v \in \mathbb{R}^{d_z}$ with label $y_v \in [C]$.
\item[Link prediction.] Each example is an edge embedding
      $z_u \odot z_v \in \mathbb{R}^{d_z}$ (Hadamard product of endpoint
      embeddings), with a binary edge-vs-non-edge label encoded as a one-hot
      vector in $\{0,1\}^{2}$ per example (so the aggregated label matrix is $Y_s \in \{0,1\}^{n_s \times 2}$). This puts LP on the same
      $(Z_s, Y_s)$ interface as NC and GC; sampling details are deferred to
      Appendix~\ref{ap:real-world-exp-setup}.
\item[Graph classification.] Each example is an embedding
      $\bar{z}_G = \frac{1}{|V_G|}\sum_{v \in V_G} z_v$ with a graph-level label.
\end{description}

The readout is otherwise identical across tasks: it receives support
embeddings, fits Eq.~\eqref{eq:ridge-dual}, and predicts query labels. This
gives \modelname{} a single training objective across node-, edge-, and
graph-level problems, without requiring a task-specific decoder architecture.

\subsection{Episodic Meta-Training}
\label{sec:meta_training}

Rather than pre-training on self-supervised or contrastive-based approaches under the 
assumption of downstream task alignment, \modelname{} is trained on the same 
few-shot prediction problem it will face at test time. Each training step proceeds as follows:

\begin{recipebox}{\modelname{}: episodic meta-training for GFMs}
\small

\noindent\textbf{Input:} A collection of source graphs, task types $\mathcal{T}=\{\text{NC},\text{LP},\text{GC}\}$, shots $K$, queries $Q$, and encoder $f_\theta$.
\smallskip

\noindent\textbf{Step 1: Sample an episode.} Select a task type (NC, LP or GC) and a source graph. Construct a support set with $K$ labeled examples per class and a disjoint query set with $Q$ examples per class.
\smallskip

\noindent\textbf{Step 2: Encode the episode.} Compute node embeddings with $f_\theta$ and build task-specific example embeddings as described in Section~\ref{sec:meta_unification}.
\smallskip

\noindent\textbf{Step 3: Fit the readout.} Solve the ridge problem in Eq.~\ref{eq:ridge-dual} on the support set.
\smallskip

\noindent\textbf{Step 4: Meta-update the encoder.} Evaluate the ridge head on the query set, compute the cross-entropy loss, and backpropagate through the closed-form solve to update $\theta$.
\end{recipebox}

The only training signal is the episodic query loss. As a result, the encoder is
optimized directly for what matters at evaluation time: producing embeddings from
which a small support set can induce an accurate linear decision rule. This
training--evaluation alignment is the key difference from post-hoc
readouts on frozen GFM embeddings.

\paragraph{\modelname{}++.} We also evaluate \modelname{}++, which consumes one episode per task type (NC, LP, GC) within the same gradient update via gradient accumulation~\citep{yang2022efficient}, rather than sampling a single task per step. This trades roughly 3$\times$ per-step compute for a more balanced per-task signal (Section~\ref{sec:experiments}).

\paragraph{Pre-training costs in Graph Foundation Models.} Link-prediction pre-training~\citep{xia2024anygraph,liu_one_2024,huang_prodigy_2023,liu2023graphprompt} samples $N_{\text{neg}}$ negatives per positive edge, giving per-epoch cost $\mathcal{O}(|E|N_{\text{neg}}d_z)$; contrastive pre-training~\citep{you_graph_2021} scales similarly with in-batch negatives. Each \modelname{} step instead processes one episode with $n_s=KC$ support and $QC$ query examples, fits the ridge readout in closed form at $\mathcal{O}(n_s^3+n_s^2d_z)$, and uses negatives only for link-prediction episodes. Because $n_s \ll |E|$ and node-/graph-classification episodes require no negatives, the cost is dominated by one encoder pass over $\mathcal{O}((K+Q)C)$ examples, explaining the 8$\sim$27$\times$ wall-clock speedup reported in Section~\ref{sec:experiments}.

\section{Experiments}
\label{sec:experiments}

Our evaluation proceeds along two complementary axes. First, we use controlled
synthetic experiments to verify the geometric failure modes identified in
Section~\ref{sec:lp-unify} under settings where the task structure is fully
known (\S\ref{sec:exp-synthetic}). Second, we evaluate \modelname{} on
real-world node classification, link prediction, and graph classification
benchmarks, comparing both predictive performance and training efficiency
against existing GFM baselines (\S\ref{sec:exp-real-graphs}).


\subsection{Controlled Validation of the Prototype Pitfalls}
\label{sec:exp-synthetic}

We present two controlled synthetic experiments that validate the first two pitfalls from Section~\ref{sec:lp-unify}: origin anchoring (\S\ref{sec:origin-anchored-failure}) and convex-hull inclusion (\S\ref{sec:convex-hull}). In each case, we compare the prototype classifier against a ridge readout with bias. The third pitfall, calibration failure (\S\ref{sec:calibration}), is validated separately in Appendix~\ref{ap:calibration-details}; full setup details and additional variants are in Appendix~\ref{ap:additional-synth}.

\paragraph{Origin-anchored failure under translation.}
We verify Remark~\ref{rem:translation} on a balanced binary
classification task in $\mathbb{R}^{64}$. We generate two Gaussian clusters
with fixed separation and translate the embedding cloud by $t u$, where
$u$ is the cluster-separation direction. Because translation preserves pairwise
distances, the task does not become harder. Figure~\ref{fig:collapse_and_convex_hull}a
shows that a bias-enabled ridge readout maintains near-perfect accuracy
throughout, while prototype accuracy degrades toward chance as the
origin-anchored boundary becomes increasingly misaligned with the translated
data cloud.

%
%

\paragraph{Convex-hull prototype failure under multimodality.}
We validate the $\varepsilon$-relaxed form of Proposition~\ref{prop:domination} (Definition~\ref{def:eps-domination}, Appendix~\ref{ap:eps-domination}) on a three-class problem in $\mathbb{R}^{64}$. One class is bimodal with modes separated by $\delta$; the other two are unimodal with prototypes at $\pm 0.6\delta$ along the same axis, plus a small offset that keeps the task linearly separable. As $\delta$ grows, the unimodal prototypes gain norm while the bimodal prototype remains bounded, so the bimodal class becomes $\varepsilon$-dominated under inner-product scoring. Figure~\ref{fig:collapse_and_convex_hull}b shows its prototype recall collapsing to zero, while ridge maintains high recall. The failure therefore comes from the single-centroid readout, not task difficulty.

\begin{figure}[htbp]
\centering
\includegraphics[width=0.92\linewidth]{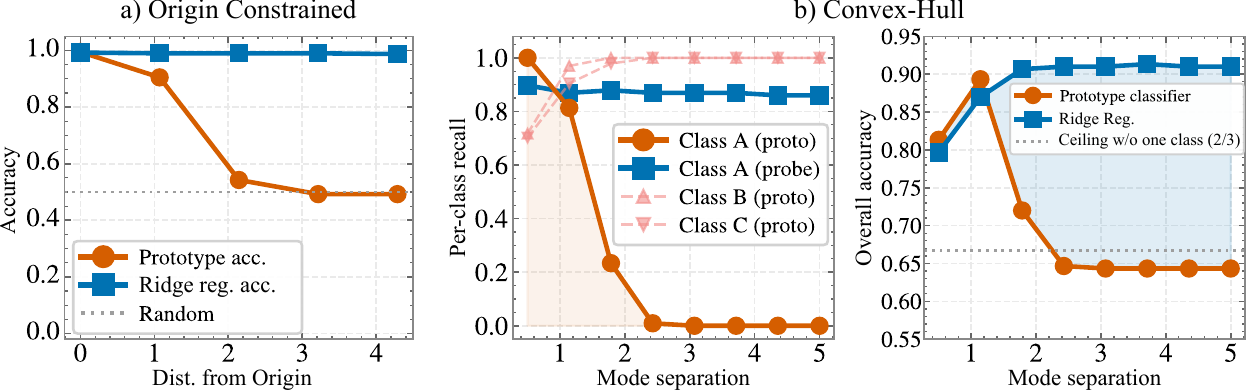}
\caption{Synthetic validation of Remark~\ref{rem:translation} and Proposition~\ref{prop:domination}. \textbf{(a)}~Under translation along the separation direction, ridge with bias stays near-perfect while prototype accuracy degrades toward chance. \textbf{(b)}~\textbf{Left:} prototype recall of the multimodal class drops to zero once $p_A$ enters the remaining-prototype hull. \textbf{Right:} overall accuracy stays high for ridge, as the task remains linearly recoverable.}
\label{fig:collapse_and_convex_hull}
\end{figure}

\subsection{Real-world Datasets and Baseline Comparison}
\label{sec:exp-real-graphs}

\paragraph{Experimental setup.} We evaluate \modelname{} and \modelname{}++ against six baselines across four families of cross-graph transfer methods: a vanilla GNN (GIN), a contrastive method (GraphCL), three GFMs (AnyGraph, GFT, GPF), and a prompt-based approach (GraphPrompt). All methods are evaluated in a $k$-shot ($k \in \{4,8,16,32,64,512\}$) setting across 20 evaluation tasks: node classification (5), link prediction (10), and graph classification (5), with performance averaged over 3 random seeds. Dataset names, training pool, and per-dataset numbers are in Appendices~\ref{ap:real-world-exp-setup}--\ref{ap:perf-tables}.

\paragraph{Convex-hull prototype analysis.}
To test whether the convex-hull failure arises in real-world GFM embeddings, we compute the per-class AnyGraph prototype on \textsc{Photo}, \textsc{Computers}, and \textsc{Cora}, and measure each prototype's distance to the convex hull of the remaining ones in the full embedding space (Appendix~\ref{ap:convex-hull-score}). Prototypes with distance smaller than any hull-defining prototype are flagged as especially susceptible. Figure~\ref{fig:datasets_domination} visualizes the geometry via 2D PCA (distances computed before projection). Across all three datasets, multiple prototypes lie substantially inside the hull, confirming that the failure mode of \S\ref{sec:convex-hull} occurs in real-world embeddings.

\begin{figure}[htbp]
    \includegraphics{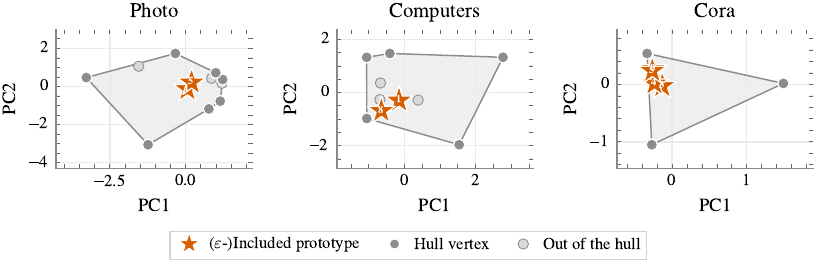}
    \caption{Convex-hull prototype geometry in AnyGraph embeddings on \textsc{Photo}, \textsc{Computers}, \textsc{Cora}. Points are class prototypes projected to 2D via PCA (distance scores computed in the full embedding space). Filled markers lie on the PCA hull; lighter markers are interior; stars flag prototypes with small convex-hull distance. Interior prototypes appear across all three datasets, evidencing the failure mode of \S\ref{sec:convex-hull} in real-world GFMs.}
    \label{fig:datasets_domination}
\end{figure}

\paragraph{Post-hoc linear readouts partially recover AnyGraph's frozen embeddings.}
As a diagnostic, we fit a linear head on frozen AnyGraph embeddings and compare it with AnyGraph's prototype readout and \modelname{}++ (Appendix~\ref{ap:posthoc-linear}, Table~\ref{tab:decision-boundary-acc}). This narrows the prototype gap but does not change the conclusion: \modelname{}++ matches or exceeds the stronger AnyGraph variant on 3 out of 5 NC datasets, with a much lighter training and inference pipeline.

\begin{wrapfigure}{r}{0.48\textwidth}
\centering
\includegraphics[width=\linewidth]{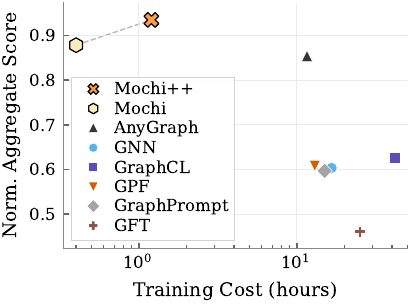}
\caption{Training cost (hours) vs.\ performance. \modelname{}/\modelname{}++ achieves competitive aggregate performance at a fraction of baselines' training cost by aligning pre-training with inference.}
\label{fig:pareto_plot}
\end{wrapfigure}

\paragraph{Training cost versus performance.} To evaluate the practical efficiency of each method, we plot training cost (wall-clock hours) against normalized aggregate performance across all tasks. 

We compute a normalized aggregate score per method by averaging its normalized performance across all $k$-shots for node classification, graph classification, and link prediction benchmarks (metric definition in Appendix~\ref{ap:normalized-aggregate}). Figure~\ref{fig:pareto_plot} displays the resulting Pareto plot on a log-scaled cost axis. \modelname{} achieves competitive aggregate performance at a fraction of the training cost of all baselines (0.4 GPU-hours compared to 11--42 hours for alternatives, a $\sim$27$\times$ speedup over the strongest baseline AnyGraph). The \modelname{}++ variant, which processes one episode per task type in each optimizer step, yields further performance gains at a modest cost increase (1.28 hours, still $\sim$8.6$\times$ faster than AnyGraph) and remains an order of magnitude cheaper than the next-best method. This result underscores a key advantage of our meta-learning-based paradigm: our method avoids the expensive link-prediction pre-training based on negative samples by aligning the pre-training task with the inference procedure, enabling \modelname{}/\modelname{}++ to be more sample efficient.

\paragraph{Head-to-head comparison with AnyGraph.}
Because AnyGraph is the strongest baseline, we provide a direct pairwise
comparison in Figure~\ref{fig:anygraph_radar}. \modelname{}++ leads AnyGraph on
NC, remains competitive on GC, and is a close second on LP, where AnyGraph is
advantaged by pre-training on 15 LP source graphs. This supports our main
conclusion: \modelname{}++ achieves the most consistent performance across
heterogeneous task families, rather than excelling on one family at the expense
of others. Moreover, \modelname{} achieves this using a disjoint multi-task pool
(4 NC, 15 LP, and 4 GC source graphs; Appendix~\ref{ap:real-world-exp-setup}),
with features precomputed once per graph and a meta-training strategy that
avoids large-scale negative sampling, yielding substantially lower training
cost.

\begin{figure}[htbp]
\centering
\includegraphics[width=0.80\linewidth]{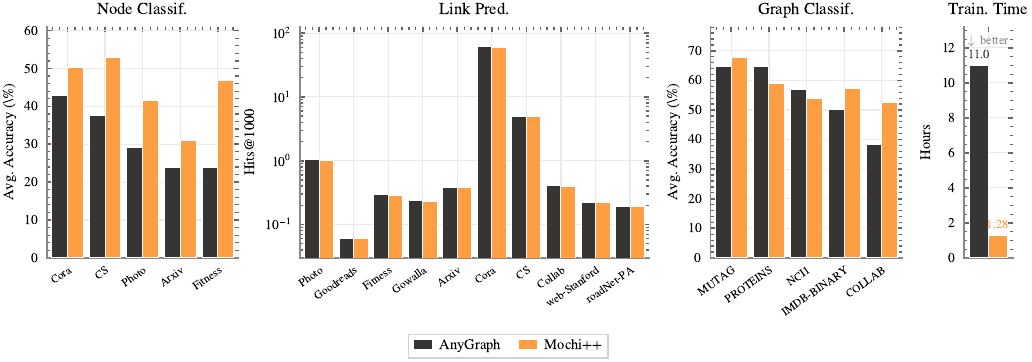}
\caption{Head-to-head of \modelname{}++ vs.\ AnyGraph across NC, LP, GC, and training efficiency. Bars show task-balanced normalized aggregates per domain. \modelname{}++ achieves the most consistent performance across tasks at $\sim$8.5$\times$ less training time.}
\label{fig:anygraph_radar}
\end{figure}

\section{Related Work}
\label{sec:related-work}

\paragraph{Graph Foundation Models.}
Graph foundation models learn representations that transfer across domains, feature spaces, and task granularities~\citep{mao_position_2024,wang2025graph,wang2025graph_2} through diverse pretraining signals: edge or relation prediction~\citep{sun_gppt_2022,liu2023graphprompt,xia2024anygraph,xia_opengraph_2024}, contrastive learning, masked reconstruction, diffusion, structural vocabularies, and task trees~\citep{you_graph_2021,davies_topology_2024,tang_cross-domain_2024,wang2024gft,wang_towards_2025}. Language-assisted variants align attributes and label spaces via textual descriptions~\citep{liu_one_2024,he2025unigraph,li_zerog_2024,chen2024llaga}; prompt and in-context methods adapt pretrained encoders via learned prompts or task-specific reductions~\citep{sun_all_2023,huang_prodigy_2023,fang_universal_2024,bevilacqua_holographic_2025}; and task-specialized GFMs target cross-graph node classification via analytical attention~\citep{zhao_fully-inductive_2025} and cross-graph link prediction via in-context learning~\citep{dong_universal_2024}. Across these approaches, representation pretraining is typically separated from the few-shot decision rule used at evaluation time.

\paragraph{Task Unification and Prototype Readouts.}
Task unification reduces node-, edge-, and graph-level prediction to a common
inference interface~\citep{liu_one_2024,huang_prodigy_2023,sun_gppt_2022,liu2023graphprompt,wang2024gft,wang_towards_2025,bevilacqua_holographic_2025},
typically via a similarity rule or prototype classifier. Prototype
readouts are attractive for their parameter efficiency, but impose a restricted
geometry~\citep{snell2017prototypical,hou2022closer}. This paper studies the
consequences of that restriction for GFMs, where structural pretraining,
multimodal classes, and norm shifts interact with prototype readouts.

\paragraph{Meta-Learning for Task-Adaptive Readouts.}

Differentiable closed-form solvers and convex optimization layers provide
task-specific classifiers while preserving end-to-end meta-learning~\citep{vilalta2002perspective,vinyals2016matching,bertinetto2018meta,lee2019meta}.
In the graph domain, several works apply meta-learning to GNN backbones for
few-shot node classification~\citep{zhou2019metagnn,huang2020gmeta,wang2022tent,liu2024metagps};
Graph Prototypical Networks~\citep{ding2020gpn} combine meta-learning
with prototype readouts, closest in spirit to \modelname{}; the differences are
that \modelname{} replaces the prototype rule with a differentiable ridge
readout and unifies across node, link, and graph-level episodes.

\section{Conclusion}
We revisited task unification in Graph Foundation Models through the downstream
inference rule. Prototype classifiers provide a convenient pairwise scoring
interface across heterogeneous tasks, but introduce structural limitations:
origin-anchored decision boundaries, convex-hull failure under unfavorable class
geometry, and norm-dependent miscalibration. Post-hoc linear readouts can
partially mitigate these issues, but cannot undo representations shaped by a
mismatched pre-training objective.

\paragraph{Limitations and broader impacts.}
\modelname{}'s efficiency relies on a frozen SVD plus hop-propagation feature
stack, a manually selected meta-training pool, and a dual-form ridge solve best
suited to the $n_s \ll d_z$ regime (see Appendix~\ref{ap:limitations} for a detailed discussion). 
\modelname{} is under the same broader-impact concerns as previous GFMs trained or evaluated on sensitive
graph data.

\modelname{} addresses the mismatch between pre-training and few-shot inference
by moving adaptation into pre-training: the encoder is trained episodically
across node-, edge-, and graph-level tasks with a differentiable ridge readout
in the loop. The result is a GFM with consistent performance across node-, link-,
and graph-level tasks, leading existing baselines on the overall normalized
aggregate while substantially reducing training cost.

\bibliographystyle{unsrtnat}
\bibliography{refs.bib}

\newpage
\appendix
\section{Why Link-Prediction Embeddings Trigger Prototype Failure}
\label{ap:lp-geometry}

The origin-anchored failure identified in Section~\ref{sec:origin-anchored-failure} is not merely a theoretical possibility; the geometry of link-prediction embeddings makes it the expected outcome. Link-prediction training with an inner-product decoder
$\hat{A}_{uv} = \sigma(z_u^\top z_v)$ produces embeddings with three
well-documented geometric properties
\citep{qiu2018network}:

\begin{enumerate}[leftmargin=*,itemsep=3pt]
\item \textbf{Norms encode degree.} For matrix-factorization-based models,
      $\|z_v\| \approx \sqrt{d_v}$, where $d_v$ is the degree of node $v$.
      This arises because higher-degree nodes must produce larger inner
      products with more neighbors.

\item \textbf{Neighbors align directionally.} Connected nodes have similar
      angular positions in embedding space, since
      $z_u^\top z_v = \|z_u\|\,\|z_v\|\cos\theta_{uv}$ must be large for
      edges $(u,v) \in E$.

\item \textbf{Embeddings cluster in a cone.} For assortative graphs
      (where connected nodes share properties), properties~(i) and~(ii)
      together concentrate embeddings in a low-dimensional cone within a
      half-space. The embeddings rarely cross the origin.
\end{enumerate}

Property~(iii) means that embeddings typically lie in a half-space, so
the origin-constrained prototype boundary is misaligned with any task whose natural
boundary does not pass through the origin (Remark~\ref{rem:translation}).

Property~(i) makes this particularly damaging for degree-correlated tasks.
If class identity correlates with node degree (e.g., distinguishing hub
nodes from peripheral nodes), the true decision boundary is a threshold on
$\|z\|$, which corresponds to a sphere centered at the origin---or, in
projection, a hyperplane at offset $\delta > 0$. The prototype boundary at
the origin cannot represent this threshold.

\section{Proof of Proposition~\ref{prop:domination}}
\label{ap:proof-domination}

\begin{proof}
Since
\[
p_c \in \mathrm{conv}\bigl(\{p_{c'}\}_{c' \neq c}\bigr),
\]
there exist coefficients $\lambda_{c'} \ge 0$ with
$\sum_{c' \neq c} \lambda_{c'} = 1$ such that
\[
p_c = \sum_{c' \neq c} \lambda_{c'} p_{c'}.
\]
Therefore, for any $z \in \mathbb{R}^{d_z}$,
\[
\langle z, p_c \rangle
=
\left\langle z, \sum_{c' \neq c} \lambda_{c'} p_{c'} \right\rangle
=
\sum_{c' \neq c} \lambda_{c'} \langle z, p_{c'} \rangle.
\]
The right-hand side is a convex combination of the scores
$\{\langle z, p_{c'} \rangle\}_{c' \neq c}$, and is therefore bounded above by
their maximum:
\[
\langle z, p_c \rangle
\le
\max_{c' \neq c} \langle z, p_{c'} \rangle.
\]
Hence class $c$ cannot be the unique argmax of the prototype scores.
\end{proof}

\section{Origin-Anchored Failure Example}
\label{ap:prototype-collapse-example}

\paragraph{Illustrative Example.}
Consider four nodes with 1D embeddings forming two balanced classes:
\[
\text{Class 0: } z_1 = -1,\; z_2 = 0 \qquad \text{Class 1: } z_3 = 1,\; z_4 = 2.
\]
The prototypes are $p_0 = -0.5$ and $p_1 = 1.5$, and the prototype decision
boundary $\langle z,\; p_0 - p_1 \rangle = 0$ reduces to $-2z = 0$, i.e.,
$z = 0$. This correctly separates the classes (all of class~0 has $z \leq 0$,
all of class~1 has $z \geq 1$), giving \textbf{100\% accuracy}.

Now translate every embedding by $t = 5$:
\[
\text{Class 0: } z_1' = 4,\; z_2' = 5 \qquad \text{Class 1: } z_3' = 6,\; z_4' = 7.
\]
The task is identical: class~0 and class~1 have the same separation, the same
margin, and the same balance. A $k$-NN classifier or a linear probe with bias
(boundary at $z' = 5.5$) still achieves perfect accuracy.

The prototypes become $p_0' = 4.5$ and $p_1' = 6.5$. The prototype gap is
unchanged: $|p_1' - p_0'| = |p_1 - p_0| = 2$. However, the prototype
decision boundary is still $z' = 0$, which is now far from the data. Since all
shifted embeddings satisfy $z' \geq 4 > 0$, the prototype score
$z' \cdot p_1' > z' \cdot p_0'$ holds for every node (both prototypes are
positive, and $p_1' > p_0'$). The prototype classifier predicts class~1 for
all nodes, yielding \textbf{50\% accuracy} on a trivially separable task.

This is the origin constraint of
Remark~\ref{rem:translation}: the prototype gap is preserved, the task
difficulty is preserved, but prototype accuracy collapses because its decision
boundary cannot shift with the data.

\section{Convex-Hull Inclusion Example}
\label{ap:multimodal-collapse-example}

\paragraph{Illustrative Example.}
Consider three class prototypes
\[
p_A = (2, 2), \quad p_B = (1, 1), \quad p_C = (0, 0).
\]
Since $p_B = \tfrac{1}{2}(p_A + p_C)$, $p_B \in \mathrm{conv}(\{p_A, p_C\})$. By Proposition~\ref{prop:domination}, $\langle z, p_B\rangle$ is never the strict argmax of the prototype scores: for any query with $z_1 + z_2 > 0$ the score $\langle z, p_A\rangle = 2(z_1+z_2)$ exceeds $\langle z, p_B\rangle = z_1 + z_2$, and for any query with $z_1 + z_2 < 0$ the score $\langle z, p_C\rangle = 0$ does. The prototype classifier therefore never predicts class~B, regardless of the underlying samples.

A linear probe with bias, in contrast, is not constrained to centroid-based scoring. Three independent linear scores with intercepts can place each prototype as the unique argmax of its own class (e.g., $\mathrm{logit}_A(z) = \tfrac{z_1+z_2}{2} - 1$, $\mathrm{logit}_B(z) = \tfrac{1}{2}$, $\mathrm{logit}_C(z) = 1 - \tfrac{z_1+z_2}{2}$). The failure is a property of the prototype rule, not of the underlying class geometry.

\section{\texorpdfstring{$\varepsilon$-inclusion}{epsilon-inclusion}}
\label{ap:eps-domination}

Exact inclusion (Proposition~\ref{prop:domination}) is a boundary case. We
introduce a soft variant that captures near-inclusion.

\begin{definition}[$\varepsilon$-Inclusion]
\label{def:eps-domination}
Class $c$ is $\varepsilon$-included if there exist $\lambda_{c'} \geq 0$ with
$\sum_{c' \neq c}\lambda_{c'} = 1$ such that
$\|p_c - \sum_{c' \neq c} \lambda_{c'} p_{c'}\| < \varepsilon$.
\end{definition}

\begin{proposition}
\label{prop:eps-domination}
If class $c$ is $\varepsilon$-included, then for all $z$ with $\|z\| \leq R$:
$\max_{c' \neq c}\langle z, p_{c'}\rangle - \langle z, p_c\rangle
\geq -\varepsilon R$.
That is, class $c$ can win the argmax by at most margin $\varepsilon R$, which
vanishes as $\varepsilon \to 0$.
\end{proposition}

\begin{proof}
By Definition~\ref{def:eps-domination}, $p_c = \sum_{c' \neq c} \lambda_{c'} p_{c'} + \delta$ for some $\|\delta\| < \varepsilon$ with $\sum_{c' \neq c} \lambda_{c'} = 1$ and $\lambda_{c'} \geq 0$. Then
$\max_{c' \neq c}\langle z, p_{c'}\rangle - \langle z, p_c\rangle = \max_{c' \neq c}\langle z, p_{c'}\rangle - \sum_{c' \neq c} \lambda_{c'} \langle z, p_{c'}\rangle - \langle z, \delta\rangle \geq -\langle z, \delta\rangle \geq -\|z\|\,\|\delta\| > -R\varepsilon$,
where the first inequality follows from the convex combination being bounded by the maximum and the second from the Cauchy-Schwarz inequality.
\end{proof}

\section{Additional Details on Synthetic Experiments}
\label{ap:additional-synth}

This appendix provides the full setup details and supplementary synthetic
results omitted from Section~\ref{sec:experiments} for space.

\subsection{Origin-Anchored Failure Under Translation}
\label{ap:translation-details}

\paragraph{Setup.}
We verify Remark~\ref{rem:translation} with a controlled binary
classification experiment. We generate $n=800$ embeddings in
$\mathbb{R}^{64}$ as two balanced Gaussian clusters separated by a fixed margin
$\Delta$. Let $u$ denote the unit vector along the cluster-separation axis. We
then translate the entire embedding cloud by $t\,u$ for $t \in [0,5]$.

This transformation preserves all pairwise distances, and therefore leaves the
underlying task geometry unchanged. In particular, a distance-based classifier
such as $k$-NN would have identical accuracy for every value of $t$. We compare
the prototype classifier against a ridge readout with bias.

\paragraph{Results.}
As shown in Figure~\ref{fig:collapse_and_convex_hull}a, the ridge readout maintains
near-perfect accuracy throughout, absorbing the translation into its intercept.
In contrast, prototype accuracy degrades monotonically toward $50\%$ as $t$
increases. This is exactly the behavior predicted by
Remark~\ref{rem:translation}: the translated data cloud moves farther
from the origin while the origin-anchored prototype boundary remains fixed.
A handcrafted two-dimensional example is provided in
Appendix~\ref{ap:prototype-collapse-example}.

\subsection{Multimodal Convex-Hull Inclusion Details}
\label{ap:multimodal-details}

\paragraph{Setup.}

We verify Proposition~\ref{prop:domination} in its $\varepsilon$-relaxed form (Definition~\ref{def:eps-domination}) on a three-class problem in $\mathbb{R}^{64}$. Class~A is bimodal, with 50 samples per mode centered at $(\mp\delta, 1, 0, \ldots, 0)$. Classes~B and~C each consist of 100 unimodal samples centered at $(\mp 0.6\delta, -0.5, 0, \ldots, 0)$. All samples include isotropic Gaussian noise.

The class prototypes are
\[
p_A = (0, 1, 0, \ldots), \quad p_B = (-0.6\delta, -0.5, 0, \ldots), \quad p_C = (+0.6\delta, -0.5, 0, \ldots),
\]
so $p_A$ stays fixed and $p_B, p_C$ drift along the $x$-axis as $\delta$ grows. The $y$-offset (dim~1) gives A a different mean from B, C, keeping the task linearly separable. We parameterize the \emph{mode separation} by $\delta \ge 0$. At $\delta = 0$, the three Gaussians coincide and predictions are at chance. As $\delta$ grows, A's modes (dominant component along $\pm x$) develop ever-larger inner products with $p_B$ (mode~1) and $p_C$ (mode~2), eventually overwhelming $\langle z, p_A\rangle$ which is bounded by the constant $y$-projection: A becomes $\varepsilon$-dominated by $\{p_B, p_C\}$.

For each value of $\delta$, we compute:
\begin{enumerate}[leftmargin=*,itemsep=2pt]
    \item the distance from $p_A$ to $\mathrm{conv}(p_B,p_C)$, corresponding
    to the $\varepsilon$ in Definition~\ref{def:eps-domination};
    \item per-class recall of the prototype classifier;
    \item per-class recall of a ridge readout with bias; and
    \item overall accuracy of both methods.
\end{enumerate}
We use a 50/50 support/query split and average all results over 20 random
seeds.

\paragraph{Results.}

As $\delta$ increases, A's per-mode samples develop dominant inner products with $p_B$ (mode~1) and $p_C$ (mode~2), while $\langle z, p_A\rangle$ stays bounded by the $y$-projection.

Figure~\ref{fig:collapse_and_convex_hull}b shows that prototype recall for class A then drops
sharply, reaching exactly $0\%$ in the fully dominated regime, while classes B
and C maintain high recall. Consequently, overall prototype accuracy plateaus
near $66\%$: effectively random on the dominated class and nearly perfect on
the remaining two.

The ridge readout, in contrast, is not constrained to class centroids. Class A
recall remains between roughly $80$--$85\%$ across all values of $\delta$,
showing that the task remains linearly recoverable even when the prototype
classifier fails completely. The growing gap between ridge and prototype
performance in the dominated regime confirms the practical relevance of
Proposition~\ref{prop:domination}.

\subsection{Calibration Analysis}
\label{ap:calibration-details}

\paragraph{Goal.}
This experiment complements Section~\ref{sec:calibration} by illustrating two
structural sources of calibration error in prototype softmax scores and
comparing them against post-hoc alternatives.

\begin{figure}[t]
\centering
\includegraphics[width=0.82\linewidth]{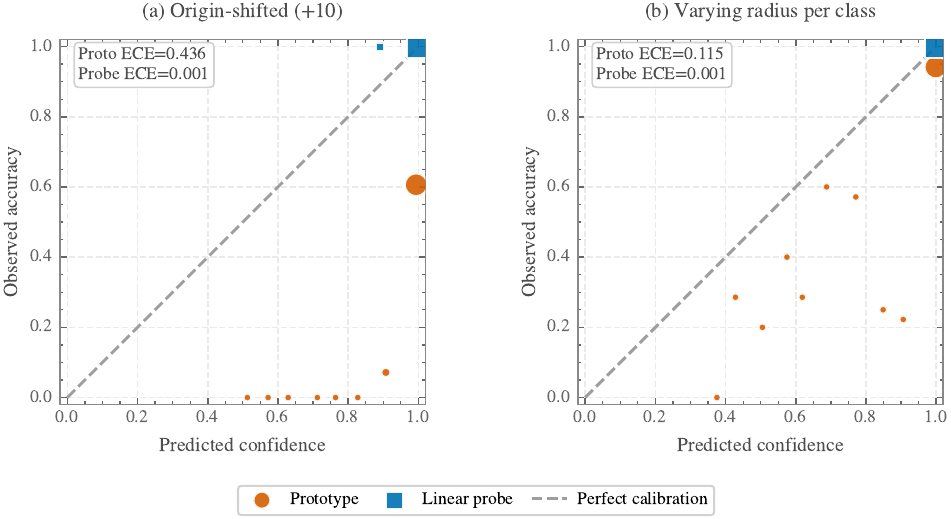}
\caption{Calibration analysis. Reliability diagrams for prototype softmax,
temperature-scaled prototypes, and ridge logistic regression under two
synthetic geometries: origin-shifted embeddings (left) and varying-radius
embeddings (right). The diagonal indicates perfect calibration; ECE is reported
in each panel.}
\label{fig:calibration}
\end{figure}

\paragraph{Setup.}
We evaluate calibration under two embedding geometries designed to isolate the
mechanisms identified in Section~\ref{sec:calibration}:
\begin{itemize}[leftmargin=*,itemsep=2pt]
    \item \textbf{Origin-shifted.} All embeddings are translated by a fixed vector $t$, so both queries $z$ and prototypes $p_c$ become $z+t$ and $p_c+t$. Expanding the post-translation inner product, $\langle z+t, p_c+t\rangle = \langle z, p_c\rangle + \langle z, t\rangle + \langle t, p_c\rangle + \|t\|^2$, the terms $\langle z, t\rangle$ and $\|t\|^2$ are class-independent and cancel in the softmax, while $\langle t, p_c\rangle$ acts as a class-specific additive bias. This stresses bias-induced miscalibration: a single global temperature cannot subtract a different bias per class.
    \item \textbf{Varying-radius.} Classes are placed at different distances
    from the origin, so that prototype norms differ substantially across
    classes. This stresses the prototype-norm bias discussed in
    Section~\ref{sec:calibration}.
\end{itemize}

For each geometry, we generate $n=500$ samples across $C=5$ classes in
$\mathbb{R}^{64}$. We report:
\begin{enumerate}[leftmargin=*,itemsep=2pt]
    \item Expected Calibration Error (ECE) of prototype softmax probabilities;
    \item ECE after temperature scaling, with the temperature fitted on the
    support set; and
    \item ECE of ridge logistic regression probabilities.
\end{enumerate}
Reliability diagrams are computed using 15 equally spaced confidence bins, and
all results are averaged over 20 random seeds.

\paragraph{Results.}
Under origin-shifted embeddings, prototype softmax is severely miscalibrated (ECE $=0.44$): the per-class additive bias $\langle t, p_c\rangle$ shifts logits across classes inconsistently, producing systematically over- or under-confident predictions. Temperature scaling reduces the error to approximately $0.22$ but cannot fully correct it: a single scalar temperature applies uniformly to all logits, while the introduced bias differs per class. Ridge logistic regression, which learns its own per-class intercept on the support set, achieves ECE $=0.001$.

Under varying-radius embeddings, prototype miscalibration is smaller in
magnitude but structurally similar: classes with larger-norm prototypes are
over-preferred even when the query is far from their true decision region.
Prototype ECE is $0.12$, temperature scaling improves it to approximately
$0.08$, and ridge logistic regression achieves near-perfect calibration
(ECE $=$ $0.001$). The reliability diagrams in Figure~\ref{fig:calibration} show
that ridge logistic regression tracks empirical accuracy much more closely
across all confidence bins.

\section{Limitations}
\label{ap:limitations}

We discuss the main limitations of \modelname{}, complementing the analysis of prototype failure modes in Section~\ref{sec:lp-unify}.

\paragraph{Fixed structural feature stack.} The encoder operates on a frozen SVD of the symmetrically normalized adjacency concatenated with $\ell{=}3$ hop propagations (Appendix~\ref{ap:real-world-exp-setup}); only the GAMLP head is meta-trained. This precomputation is what makes training cheap, but it also caps the structural information available to the encoder: graphs whose discriminative signal is poorly captured by low-rank spectral and short-range propagated features cannot be fully recovered downstream. Replacing the stack with a deeper trainable backbone would address this at the cost of the wall-clock advantage reported in Section~\ref{sec:exp-real-graphs}.

\paragraph{Composition of the meta-training pool.} \modelname{} is meta-trained on a manually selected pool of source graphs (4 NC, 15 LP, 4 GC; Appendix~\ref{ap:real-world-exp-setup}). Performance on a downstream graph plausibly depends on the pool covering its structural and feature regime; the paper does not characterize transfer to graphs that fall outside the regimes present in this pool, nor does it study automated pool selection.

\paragraph{Episode-hyperparameter range.} Meta-training uses support sizes $k\in[8,32]$ and query sizes $q\in[16,64]$ per class (Appendix~\ref{ap:real-world-exp-setup}), while evaluation extends to $k{=}512$. Behavior at support sizes much larger or smaller than the meta-training range is implicit rather than directly studied, and the ridge regularizer $\lambda$ is fixed across episodes.

\section{Additional Details on Real-World Dataset Experiments}

\subsection{Experimental Setup}
\label{ap:real-world-exp-setup}

\paragraph{Baselines and evaluation protocol.}
We compare against a diverse set of baselines: a vanilla GNN encoder (GIN), a contrastive pre-training method (GraphCL), three graph foundation models (AnyGraph, GFT, GPF), and a prompt-based method (GraphPrompt), which we treat as a foundation-model baseline by fine-tuning a single prompt across all training datasets. We evaluate transfer across three task families: node classification (NC), link prediction (LP), and graph classification (GC). For NC and GC, we vary the number of labeled support examples as $k \in \{4,8,16,32,64,512\}$; for LP, we use an 80/20 train/test split. All reported results are averaged over three random seeds, and evaluation datasets are held out from training for all methods.

\paragraph{Evaluation datasets.}
Our evaluation suite contains 20 held-out tasks: 5 node classification (Cora, CS, Photo, Arxiv, Fitness), 10 link prediction (Photo, Goodreads, Fitness, Gowalla, Arxiv, Cora, CS, Collab, web-Stanford, roadNet-PA), and 5 graph classification (MUTAG, PROTEINS, NCI1, IMDB-Binary, COLLAB).

\paragraph{Training datasets and baselines.}
\modelname{} is trained on a disjoint multi-task pool containing 4 NC datasets (CiteSeer, PubMed, Physics, Computers), 15 LP datasets (products\_tech, yelp2018, yelp\_textfeat, products\_home, steam\_textfeat, amazon\_textfeat, amazon-book, citation-2019, citation-classic, pubmed, citeseer, ppa, p2p-Gnutella06, soc-Epinions1, email-Enron), and 4 GC datasets (DD, ENZYMES, REDDIT-Multi-5K, ogbg-ppa). Episodes are sampled uniformly across task families with a 1:1:1 ratio. All baselines are trained on the same 15 LP training datasets using their respective objectives. We implement the GIN baseline using the standard GIN architecture~\citep{xu2018powerful} in PyTorch Geometric~\citep{fey2019fast}, trained as a Graph Auto-Encoder~\citep{kipf2016variational}. GraphCL is implemented with the PyGCL library~\citep{Zhu:2021tu}. For all remaining baselines, we use the authors' official implementations.

\paragraph{Encoder architecture.}
All methods share the same preprocessing. We compute node embeddings using a truncated SVD of the symmetrically normalized adjacency matrix, $\tilde{A}=D^{-1/2}AD^{-1/2}$, projected to $d=512$ dimensions with two power iterations. When node features are available, we concatenate these embeddings with an SVD of the feature matrix. From the resulting frozen SVD representation, we precompute $\ell=3$ hop adjacency propagations and concatenate them as input to a trainable GAMLP encoder~\citep{zhang2022graph}. GAMLP applies a per-hop linear projection ($512 \to 512$), node-adaptive attention over the $\ell+1$ hops using hop-0 as the attention query, a weighted sum over hops, and a residual 3-layer MLP with ReLU and dropout. This produces node embeddings $z_v \in \mathbb{R}^{512}$.

\paragraph{Ridge readout.}
On top of the encoder, \modelname{} uses a differentiable affine ridge readout following R2-D2-style closed-form solvers~\citep{bertinetto2018meta}. We augment the support embeddings with a constant feature, yielding a task-specific intercept. When regularized, this intercept is softly shrunk toward zero but is still estimated from the support set, unlike prototype classifiers whose decision boundaries remain origin-anchored. The readout solves the bias-augmented ridge problem in Eq.~\ref{eq:ridge-dual} in dual Woodbury form~\citep{petersen2008matrix}, allowing end-to-end gradient flow without inner-loop optimization. A learnable temperature scalar is applied to the ridge logits before the meta-outer-loop cross-entropy, which keeps the softmax well-conditioned despite squared-loss-fitted logits lying near $[0,1]$.

\paragraph{Episode construction and optimization.}
Episodes use variable support sizes $k \in [8,32]$ and query sizes $q \in [16,64]$ per class, with up to 64 classes. We use ridge regularization $\lambda=10$ and label smoothing of $0.1$. Training runs for 10{,}000 episodic steps with AdamW, learning rate $3{\times}10^{-4}$, weight decay $10^{-4}$, cosine annealing, and gradient clipping at norm 1.0. \modelname{}++ uses the same architecture but samples three episodes per step, one from each task family, following~\citep{yang2022efficient}.

\paragraph{Task-specific details.}
For LP episodes, positive examples are true edges and negatives are sampled uniformly at random from a fixed pool of non-edges within the same graph, using a 1:1 ratio in both support and query sets. Edge embeddings are formed by the Hadamard product of endpoint node embeddings. Positive support and query edges are masked from $A$ before the SVD and hop propagations, ensuring that the encoder does not receive the label as input. For GC, graph-level embeddings are obtained by mean-pooling node representations.

\paragraph{Compute and downstream evaluation.}
All experiments are run on a single NVIDIA RTX A6000 GPU with 48GB memory. \modelname{} trains in approximately 0.4 GPU-hours, while \modelname{}++ trains in 1.28 GPU-hours. Baseline training costs range from 11 GPU-hours for AnyGraph to more than 41 GPU-hours for GraphCL. For fair comparison, all methods are evaluated with identical frozen-embedding downstream protocols: prototypes and logistic regression for NC/GC, and dot-product scoring for LP.

\paragraph{Few-shot evaluation protocol.}
For NC and GC, we construct $k$-shot episodes by sampling $k$ support examples per class; when a class has fewer than $k$ labeled examples (notably at $k=512$ on datasets with small per-class label counts), support examples are drawn with replacement. A fixed number of query examples per class is used across all $k$ values, since query predictions are independent and the query count does not affect the trained model's decision rule. For LP, we report Hits@1000 in the standard information-retrieval sense: for each test positive edge, we rank it against the candidate-negative pool and report the fraction of test positives whose rank is at most 1000. On GC benchmarks based on TUDatasets (e.g., MUTAG with 188 graphs), reported numbers average across 3 random seeds, each using a fresh random train/test split.

\subsection{Normalized Aggregate Score}
\label{ap:normalized-aggregate}

For each evaluation cell $c$ in task family
$t \in \{\mathrm{NC}, \mathrm{LP}, \mathrm{GC}\}$, where a cell denotes a dataset
for LP and a $(\text{dataset}, k\text{-shot})$ pair for NC and GC, let
$v_{t,c}$ be a method's raw downstream score on that cell (AUC for LP, accuracy
for NC/GC). We normalize each cell by the best score achieved by any method on
that cell, $v^{*}_{t,c}$, and compute the task-family aggregate
\begin{equation}
S_t = \frac{1}{|C_t|} \sum_{c \in C_t} \frac{v_{t,c}}{v^{*}_{t,c}} .
\end{equation}
The overall score is then the macro-average across task families,
$S = (S_\mathrm{NC} + S_\mathrm{LP} + S_\mathrm{GC})/3$.

This normalized aggregate provides a simple balanced summary across heterogeneous
evaluation settings. Cell-wise normalization makes scores comparable across
metrics, such as AUC for LP and accuracy for NC/GC, while macro-averaging over
task families prevents families with more evaluation cells, such as LP, from
dominating the overall score.

\subsection{Post-hoc Linear Recovery on Frozen AnyGraph Embeddings}
\label{ap:posthoc-linear}

Table~\ref{tab:decision-boundary-acc} compares three scenarios on five representative node-classification datasets: AnyGraph with its shipped prototype readout, AnyGraph augmented with an additional linear head fit on top of its frozen embeddings (a diagnostic ablation rather than a benchmark baseline, since AnyGraph does not train or evaluate with such a readout in its original form), and \modelname{}++ with its ridge-based episodic readout. The main point is that replacing prototypes with a stronger linear head does improve AnyGraph, confirming that prototype inference leaves performance on the table. However, this gain is only partial: on Cora, CS, and Arxiv, \modelname{}++ still exceeds the post-hoc linear head fit on AnyGraph embeddings, while on Photo and Fitness the post-hoc head is competitive with or slightly better than \modelname{}++. Overall, the comparison in the paper is not altered: \modelname{} remains competitive with the actual AnyGraph system while using a substantially lighter training and inference pipeline. Following the main NC benchmark, each entry averages performance over all $k$-shot settings.

\begin{table}[htbp]
\small
\centering
\caption{Post-hoc linear recovery on frozen AnyGraph embeddings vs.\ \modelname{}++: few-shot NC accuracy averaged across all $k$-shots, using the per-cell means from the main NC benchmark (Tab.~\ref{tab:nc}). The middle column is a diagnostic ablation, since AnyGraph does not ship with a linear readout. The gain is a partial recovery over the prototype readout; \modelname{}++ remains competitive while using a lighter pipeline. Bold marks the best per row.}
\label{tab:decision-boundary-acc}
\begin{tabular}{l c c c}
\toprule
Dataset & AnyGraph (Prototype) & AnyGraph + Linear & \modelname{}++ \\
\midrule
Cora     & $42.8$ & $47.4$ & $\mathbf{50.4}$ \\
CS       & $37.5$ & $46.1$ & $\mathbf{53.0}$  \\
Photo    & $29.1$ & $\mathbf{50.8}$ & $41.6$ \\
Arxiv    & $24.0$ & $29.6$ & $\mathbf{31.0}$ \\
Fitness  & $23.8$ & $\mathbf{49.7}$ & $46.9$  \\
\bottomrule
\end{tabular}
\end{table}

\subsection{Convex-Hull Distance Score for Real-World Prototype Geometry}
\label{ap:convex-hull-score}

To quantify the convex-hull prototype failure in real-world embeddings, we
measure for each class prototype its distance to the convex hull of the
remaining prototypes. Let $\{p_1,\dots,p_C\}\subset\mathbb{R}^{d_z}$ denote the
class prototypes for a dataset. For each class $c$, we compute
\[
d_{\mathrm{CH}}(c)
=
\min_{w \in \Delta}
\left\|
\sum_{c' \neq c} w_{c'} p_{c'} - p_c
\right\|_2,
\]
where $\Delta$ is the simplex over the remaining classes:
\[
\Delta
=
\left\{
w : w_{c'} \ge 0,\;
\sum_{c' \neq c} w_{c'} = 1
\right\}.
\]
Thus, $d_{\mathrm{CH}}(c)$ is the Euclidean distance from $p_c$ to the convex
hull of the other prototypes. Smaller values indicate that the prototype lies
closer to, or inside, that hull.

For cross-dataset comparison only, we optionally normalize this quantity by the
mean pairwise prototype distance,
\[
\bar d
=
\frac{2}{C(C-1)}
\sum_{i<j}\|p_i-p_j\|_2,
\qquad
\tilde d_{\mathrm{CH}}(c)=\frac{d_{\mathrm{CH}}(c)}{\bar d}.
\]
This normalization is not essential to the geometry itself; it is used only to
make distances more comparable across datasets of different scales.

For visualization, we project the prototypes to two dimensions using PCA and
draw the outer hull in that projection. A prototype is highlighted in the
figure when its convex-hull distance score is smaller than the smallest score
among the prototypes that define that outer hull. In other words, we flag
prototypes that are more interior, in the full embedding space, than any of the
hull-defining prototypes. This criterion is used only for highlighting in the
figure; all convex-hull distances themselves are computed in the original
embedding space.

Applying this analysis to AnyGraph embeddings on \textsc{Photo},
\textsc{Computers}, and \textsc{Cora}, we find several prototypes with very
small convex-hull distances, including prototypes that are more interior than
the hull-defining classes. This provides direct evidence that the convex-hull
prototype failure is present in real-world datasets.

\section{Performance Comparison between \modelname{}/\modelname{}++ and Baselines}
\label{ap:perf-tables}

We report full performance comparison tables for \modelname{}, \modelname{}++, and all baselines across $k \in \{4,8,16,32,64,512\}$ shots. We highlight the \colorbox{c1}{first}, \colorbox{c2}{second}, and \colorbox{c3}{third} best methods per setting. Column headers abbreviate the baselines introduced in Appendix~\ref{ap:real-world-exp-setup}: GNN (vanilla GIN encoder), GraphCL, AnyG (AnyGraph), GPF, GP (GraphPrompt), and GFT\@.

For our methods, the subscripts $_L$ and $_m$ denote the readout head used at evaluation time on top of the same trained encoder. The $_L$ variant uses the linear/ridge readout used during training. This convex readout is central to our training design, since it avoids the inner-loop optimization required by non-convex task-specific heads. To test whether additional inference-time nonlinearity can improve performance, we also include an ablation, denoted by $_m$, that replaces the linear readout with a single-layer MLP at evaluation time. These results are reported in the corresponding $_m$ columns of the node classification and graph classification tables.

\paragraph{Analysis of results.}
\modelname{}++ achieves the strongest overall performance on NC and is a close second to AnyGraph on LP, where AnyGraph is explicitly optimized for link-prediction pre-training. For GC, the large standard deviations across few-shot splits prevent a clear winner from emerging. Still, \modelname{}++ shows the most stable competitive behavior: while it is not always the top method in mean accuracy, its performance is consistently close to the best-performing methods and often lies within their error margins. This indicates that \modelname{}++ provides robust transfer across graph-level tasks without requiring GC-specific specialization.

\begin{table}[h]\centering
\caption{Node classification accuracy (\%). Top-3 highlighted per row.}\label{tab:nc}
\resizebox{\columnwidth}{!}{%
\begin{tabular}{llcccccccccc}
\toprule
Data & Set & GNN & GraphCL & AnyG & GPF & GP & GFT & \modelname{}$_m$ & \modelname{}$_L$ & \modelname{}++$_m$ & \modelname{}++$_L$ \\
\midrule
\multirow{6}{*}{Cora} & 4 & $19.14{\scriptstyle \pm 1.86}$ & $18.52{\scriptstyle \pm 1.61}$ & \cellcolor{c2}$32.22{\scriptstyle \pm 0.74}$ & $18.15{\scriptstyle \pm 3.16}$ & $19.38{\scriptstyle \pm 2.47}$ & $12.47{\scriptstyle \pm 1.19}$ & $16.05{\scriptstyle \pm 2.68}$ & \cellcolor{c3}$24.94{\scriptstyle \pm 1.50}$ & $13.21{\scriptstyle \pm 2.52}$ & \cellcolor{c1}$34.94{\scriptstyle \pm 3.19}$ \\
 & 8 & $18.02{\scriptstyle \pm 5.93}$ & $15.93{\scriptstyle \pm 5.34}$ & \cellcolor{c1}$34.81{\scriptstyle \pm 3.16}$ & $18.89{\scriptstyle \pm 0.98}$ & $19.01{\scriptstyle \pm 6.17}$ & $12.47{\scriptstyle \pm 1.19}$ & $31.48{\scriptstyle \pm 1.48}$ & $32.35{\scriptstyle \pm 1.75}$ & \cellcolor{c2}$32.47{\scriptstyle \pm 5.57}$ & \cellcolor{c3}$32.47{\scriptstyle \pm 4.63}$ \\
 & 16 & $13.83{\scriptstyle \pm 2.47}$ & $15.43{\scriptstyle \pm 2.52}$ & \cellcolor{c3}$34.69{\scriptstyle \pm 3.24}$ & $13.83{\scriptstyle \pm 4.06}$ & $14.20{\scriptstyle \pm 2.38}$ & $14.20{\scriptstyle \pm 2.47}$ & $30.74{\scriptstyle \pm 1.96}$ & $33.95{\scriptstyle \pm 3.08}$ & \cellcolor{c2}$37.41{\scriptstyle \pm 5.09}$ & \cellcolor{c1}$47.16{\scriptstyle \pm 0.93}$ \\
 & 32 & $12.72{\scriptstyle \pm 2.38}$ & $10.86{\scriptstyle \pm 0.57}$ & $39.26{\scriptstyle \pm 2.80}$ & $9.63{\scriptstyle \pm 1.11}$ & $12.47{\scriptstyle \pm 1.50}$ & $11.23{\scriptstyle \pm 1.54}$ & $36.42{\scriptstyle \pm 2.47}$ & \cellcolor{c3}$39.63{\scriptstyle \pm 3.23}$ & \cellcolor{c2}$48.64{\scriptstyle \pm 4.29}$ & \cellcolor{c1}$50.37{\scriptstyle \pm 1.92}$ \\
 & 64 & $14.32{\scriptstyle \pm 2.04}$ & $15.43{\scriptstyle \pm 0.57}$ & \cellcolor{c3}$47.65{\scriptstyle \pm 3.44}$ & $11.23{\scriptstyle \pm 0.21}$ & $16.42{\scriptstyle \pm 2.88}$ & $13.58{\scriptstyle \pm 3.08}$ & $44.94{\scriptstyle \pm 4.13}$ & $46.79{\scriptstyle \pm 2.63}$ & \cellcolor{c2}$57.90{\scriptstyle \pm 4.34}$ & \cellcolor{c1}$60.99{\scriptstyle \pm 4.42}$ \\
 & 512 & $18.64{\scriptstyle \pm 6.38}$ & $17.16{\scriptstyle \pm 0.57}$ & $68.40{\scriptstyle \pm 1.90}$ & $20.12{\scriptstyle \pm 0.57}$ & $20.12{\scriptstyle \pm 1.83}$ & $13.95{\scriptstyle \pm 1.90}$ & $73.33{\scriptstyle \pm 0.37}$ & \cellcolor{c3}$73.58{\scriptstyle \pm 0.43}$ & \cellcolor{c1}$78.77{\scriptstyle \pm 0.57}$ & \cellcolor{c2}$76.30{\scriptstyle \pm 2.22}$ \\
\midrule
\multirow{6}{*}{CS} & 4 & $13.35{\scriptstyle \pm 1.95}$ & $15.80{\scriptstyle \pm 1.34}$ & \cellcolor{c3}$19.06{\scriptstyle \pm 0.64}$ & $15.49{\scriptstyle \pm 3.58}$ & $14.37{\scriptstyle \pm 1.90}$ & $9.77{\scriptstyle \pm 5.10}$ & $3.86{\scriptstyle \pm 2.08}$ & \cellcolor{c1}$25.64{\scriptstyle \pm 2.65}$ & $3.55{\scriptstyle \pm 1.34}$ & \cellcolor{c2}$25.64{\scriptstyle \pm 2.65}$ \\
 & 8 & $9.71{\scriptstyle \pm 2.44}$ & $12.69{\scriptstyle \pm 1.04}$ & \cellcolor{c3}$25.51{\scriptstyle \pm 1.73}$ & $9.44{\scriptstyle \pm 0.22}$ & $9.33{\scriptstyle \pm 1.25}$ & $8.49{\scriptstyle \pm 4.19}$ & $2.91{\scriptstyle \pm 1.55}$ & \cellcolor{c1}$32.64{\scriptstyle \pm 3.20}$ & $3.13{\scriptstyle \pm 1.80}$ & \cellcolor{c2}$32.64{\scriptstyle \pm 3.20}$ \\
 & 16 & $5.86{\scriptstyle \pm 1.05}$ & $8.11{\scriptstyle \pm 1.06}$ & \cellcolor{c3}$28.91{\scriptstyle \pm 5.23}$ & $7.57{\scriptstyle \pm 3.05}$ & $6.60{\scriptstyle \pm 1.36}$ & $6.29{\scriptstyle \pm 3.42}$ & $2.84{\scriptstyle \pm 0.98}$ & \cellcolor{c1}$43.61{\scriptstyle \pm 4.37}$ & $2.87{\scriptstyle \pm 1.31}$ & \cellcolor{c2}$43.61{\scriptstyle \pm 4.37}$ \\
 & 32 & $3.27{\scriptstyle \pm 0.24}$ & $3.82{\scriptstyle \pm 0.36}$ & \cellcolor{c3}$37.53{\scriptstyle \pm 4.56}$ & $2.96{\scriptstyle \pm 0.17}$ & $2.84{\scriptstyle \pm 1.10}$ & $9.02{\scriptstyle \pm 11.75}$ & $21.42{\scriptstyle \pm 0.19}$ & \cellcolor{c1}$60.99{\scriptstyle \pm 3.80}$ & $22.28{\scriptstyle \pm 0.56}$ & \cellcolor{c2}$60.99{\scriptstyle \pm 3.80}$ \\
 & 64 & $3.04{\scriptstyle \pm 0.35}$ & $3.82{\scriptstyle \pm 0.36}$ & \cellcolor{c3}$49.55{\scriptstyle \pm 6.06}$ & $2.20{\scriptstyle \pm 1.12}$ & $4.20{\scriptstyle \pm 0.87}$ & $2.44{\scriptstyle \pm 1.24}$ & $30.79{\scriptstyle \pm 1.73}$ & \cellcolor{c1}$73.70{\scriptstyle \pm 3.06}$ & $30.22{\scriptstyle \pm 1.11}$ & \cellcolor{c2}$73.70{\scriptstyle \pm 3.06}$ \\
 & 512 & $2.98{\scriptstyle \pm 1.99}$ & $3.78{\scriptstyle \pm 0.50}$ & $64.45{\scriptstyle \pm 6.04}$ & $2.24{\scriptstyle \pm 0.05}$ & $4.66{\scriptstyle \pm 2.43}$ & $10.07{\scriptstyle \pm 11.00}$ & \cellcolor{c1}$85.62{\scriptstyle \pm 1.54}$ & \cellcolor{c3}$83.20{\scriptstyle \pm 2.41}$ & \cellcolor{c2}$85.23{\scriptstyle \pm 1.90}$ & $81.54{\scriptstyle \pm 2.17}$ \\
\midrule
\multirow{6}{*}{Photo} & 4 & $18.24{\scriptstyle \pm 2.97}$ & $21.45{\scriptstyle \pm 2.27}$ & \cellcolor{c1}$32.42{\scriptstyle \pm 7.24}$ & $12.74{\scriptstyle \pm 1.13}$ & $23.52{\scriptstyle \pm 5.55}$ & $21.75{\scriptstyle \pm 14.71}$ & $16.01{\scriptstyle \pm 0.79}$ & \cellcolor{c2}$30.31{\scriptstyle \pm 2.33}$ & $15.94{\scriptstyle \pm 2.57}$ & \cellcolor{c3}$30.31{\scriptstyle \pm 2.33}$ \\
 & 8 & $10.17{\scriptstyle \pm 1.88}$ & $13.06{\scriptstyle \pm 2.55}$ & \cellcolor{c2}$35.95{\scriptstyle \pm 3.67}$ & $12.04{\scriptstyle \pm 0.16}$ & $13.07{\scriptstyle \pm 3.16}$ & $19.34{\scriptstyle \pm 12.51}$ & \cellcolor{c3}$34.06{\scriptstyle \pm 3.11}$ & $31.82{\scriptstyle \pm 3.12}$ & \cellcolor{c1}$36.97{\scriptstyle \pm 3.50}$ & $32.48{\scriptstyle \pm 2.01}$ \\
 & 16 & $7.71{\scriptstyle \pm 2.05}$ & $6.19{\scriptstyle \pm 1.06}$ & $21.29{\scriptstyle \pm 2.72}$ & $11.26{\scriptstyle \pm 0.23}$ & $9.53{\scriptstyle \pm 1.09}$ & $5.97{\scriptstyle \pm 2.26}$ & $29.57{\scriptstyle \pm 5.12}$ & \cellcolor{c3}$30.51{\scriptstyle \pm 3.03}$ & \cellcolor{c1}$31.71{\scriptstyle \pm 4.73}$ & \cellcolor{c2}$30.76{\scriptstyle \pm 2.77}$ \\
 & 32 & $4.32{\scriptstyle \pm 0.79}$ & $4.78{\scriptstyle \pm 0.50}$ & $21.72{\scriptstyle \pm 1.77}$ & $3.87{\scriptstyle \pm 0.67}$ & $3.45{\scriptstyle \pm 0.30}$ & $4.57{\scriptstyle \pm 0.47}$ & $34.65{\scriptstyle \pm 0.73}$ & \cellcolor{c1}$38.53{\scriptstyle \pm 2.14}$ & \cellcolor{c3}$34.89{\scriptstyle \pm 2.12}$ & \cellcolor{c2}$38.53{\scriptstyle \pm 2.14}$ \\
 & 64 & $3.89{\scriptstyle \pm 1.52}$ & $4.23{\scriptstyle \pm 0.71}$ & $28.05{\scriptstyle \pm 1.45}$ & $1.51{\scriptstyle \pm 0.04}$ & $3.35{\scriptstyle \pm 0.43}$ & $5.51{\scriptstyle \pm 4.08}$ & \cellcolor{c3}$47.11{\scriptstyle \pm 0.60}$ & \cellcolor{c1}$52.05{\scriptstyle \pm 0.47}$ & $46.58{\scriptstyle \pm 1.27}$ & \cellcolor{c2}$52.05{\scriptstyle \pm 0.47}$ \\
 & 512 & $3.34{\scriptstyle \pm 0.85}$ & $4.52{\scriptstyle \pm 0.29}$ & $34.99{\scriptstyle \pm 1.69}$ & $1.54{\scriptstyle \pm 0.04}$ & $2.32{\scriptstyle \pm 0.88}$ & $5.76{\scriptstyle \pm 6.46}$ & \cellcolor{c2}$68.60{\scriptstyle \pm 1.21}$ & \cellcolor{c3}$66.27{\scriptstyle \pm 0.81}$ & \cellcolor{c1}$68.62{\scriptstyle \pm 1.51}$ & $65.25{\scriptstyle \pm 1.20}$ \\
\midrule
\multirow{6}{*}{Arxiv} & 4 & $12.34{\scriptstyle \pm 2.59}$ & $13.30{\scriptstyle \pm 1.48}$ & \cellcolor{c3}$23.41{\scriptstyle \pm 2.98}$ & $15.38{\scriptstyle \pm 0.05}$ & $9.07{\scriptstyle \pm 0.02}$ & $11.72{\scriptstyle \pm 9.87}$ & $2.27{\scriptstyle \pm 0.35}$ & \cellcolor{c1}$29.49{\scriptstyle \pm 2.12}$ & $0.57{\scriptstyle \pm 0.09}$ & \cellcolor{c2}$23.42{\scriptstyle \pm 3.56}$ \\
 & 8 & $10.35{\scriptstyle \pm 0.85}$ & $10.24{\scriptstyle \pm 1.80}$ & \cellcolor{c1}$26.48{\scriptstyle \pm 1.65}$ & $16.15{\scriptstyle \pm 0.79}$ & $8.28{\scriptstyle \pm 0.23}$ & $4.87{\scriptstyle \pm 3.93}$ & $1.53{\scriptstyle \pm 0.25}$ & \cellcolor{c3}$25.58{\scriptstyle \pm 1.70}$ & $2.16{\scriptstyle \pm 0.36}$ & \cellcolor{c2}$26.00{\scriptstyle \pm 1.02}$ \\
 & 16 & $5.89{\scriptstyle \pm 1.20}$ & $7.46{\scriptstyle \pm 1.07}$ & \cellcolor{c2}$28.12{\scriptstyle \pm 1.89}$ & $7.65{\scriptstyle \pm 0.22}$ & $8.06{\scriptstyle \pm 0.76}$ & $2.48{\scriptstyle \pm 1.93}$ & $2.16{\scriptstyle \pm 0.16}$ & \cellcolor{c3}$26.12{\scriptstyle \pm 1.56}$ & $1.94{\scriptstyle \pm 0.27}$ & \cellcolor{c1}$29.46{\scriptstyle \pm 0.54}$ \\
 & 32 & $5.81{\scriptstyle \pm 3.64}$ & $5.29{\scriptstyle \pm 0.60}$ & \cellcolor{c3}$20.76{\scriptstyle \pm 2.87}$ & $7.94{\scriptstyle \pm 0.74}$ & $8.37{\scriptstyle \pm 0.76}$ & $3.13{\scriptstyle \pm 0.31}$ & $1.82{\scriptstyle \pm 0.54}$ & \cellcolor{c2}$28.88{\scriptstyle \pm 1.31}$ & $3.12{\scriptstyle \pm 0.58}$ & \cellcolor{c1}$29.78{\scriptstyle \pm 0.49}$ \\
 & 64 & $4.99{\scriptstyle \pm 3.94}$ & $5.14{\scriptstyle \pm 0.39}$ & \cellcolor{c3}$20.60{\scriptstyle \pm 1.84}$ & $7.78{\scriptstyle \pm 0.29}$ & $7.94{\scriptstyle \pm 1.03}$ & $3.92{\scriptstyle \pm 3.21}$ & $1.20{\scriptstyle \pm 0.18}$ & \cellcolor{c1}$35.89{\scriptstyle \pm 0.61}$ & $0.81{\scriptstyle \pm 0.08}$ & \cellcolor{c2}$35.48{\scriptstyle \pm 1.32}$ \\
 & 512 & $1.82{\scriptstyle \pm 2.72}$ & $2.56{\scriptstyle \pm 0.64}$ & \cellcolor{c3}$24.69{\scriptstyle \pm 2.82}$ & $6.08{\scriptstyle \pm 0.07}$ & $6.10{\scriptstyle \pm 0.42}$ & $1.60{\scriptstyle \pm 1.09}$ & $6.34{\scriptstyle \pm 0.07}$ & \cellcolor{c2}$38.70{\scriptstyle \pm 1.21}$ & $7.61{\scriptstyle \pm 0.29}$ & \cellcolor{c1}$41.71{\scriptstyle \pm 1.01}$ \\
\midrule
\multirow{6}{*}{Fitness} & 4 & \cellcolor{c3}$26.12{\scriptstyle \pm 9.59}$ & $24.57{\scriptstyle \pm 7.26}$ & $20.00{\scriptstyle \pm 11.45}$ & $16.03{\scriptstyle \pm 10.58}$ & $15.04{\scriptstyle \pm 5.77}$ & $9.08{\scriptstyle \pm 1.47}$ & $8.76{\scriptstyle \pm 2.44}$ & \cellcolor{c1}$41.12{\scriptstyle \pm 0.56}$ & $6.87{\scriptstyle \pm 0.48}$ & \cellcolor{c2}$40.31{\scriptstyle \pm 0.47}$ \\
 & 8 & $9.79{\scriptstyle \pm 4.77}$ & $14.16{\scriptstyle \pm 1.23}$ & $22.55{\scriptstyle \pm 5.33}$ & $9.38{\scriptstyle \pm 1.46}$ & $10.57{\scriptstyle \pm 1.11}$ & $9.03{\scriptstyle \pm 1.40}$ & \cellcolor{c3}$40.72{\scriptstyle \pm 2.83}$ & $40.32{\scriptstyle \pm 3.43}$ & \cellcolor{c2}$41.31{\scriptstyle \pm 0.77}$ & \cellcolor{c1}$41.54{\scriptstyle \pm 0.26}$ \\
 & 16 & $8.23{\scriptstyle \pm 2.72}$ & $8.82{\scriptstyle \pm 1.08}$ & $21.10{\scriptstyle \pm 1.33}$ & $5.43{\scriptstyle \pm 0.10}$ & $6.65{\scriptstyle \pm 0.67}$ & $9.03{\scriptstyle \pm 1.40}$ & \cellcolor{c2}$40.78{\scriptstyle \pm 3.11}$ & \cellcolor{c1}$42.23{\scriptstyle \pm 2.90}$ & $39.82{\scriptstyle \pm 2.16}$ & \cellcolor{c3}$40.11{\scriptstyle \pm 1.03}$ \\
 & 32 & $4.46{\scriptstyle \pm 0.39}$ & $7.23{\scriptstyle \pm 1.95}$ & $21.95{\scriptstyle \pm 3.03}$ & $5.61{\scriptstyle \pm 0.28}$ & $5.44{\scriptstyle \pm 0.88}$ & $6.59{\scriptstyle \pm 0.84}$ & \cellcolor{c2}$45.99{\scriptstyle \pm 2.21}$ & \cellcolor{c1}$46.29{\scriptstyle \pm 1.92}$ & $41.64{\scriptstyle \pm 2.37}$ & \cellcolor{c3}$43.23{\scriptstyle \pm 1.01}$ \\
 & 64 & $3.84{\scriptstyle \pm 0.67}$ & $6.60{\scriptstyle \pm 1.00}$ & $27.53{\scriptstyle \pm 2.83}$ & $9.08{\scriptstyle \pm 5.76}$ & $6.51{\scriptstyle \pm 1.76}$ & $4.80{\scriptstyle \pm 2.39}$ & \cellcolor{c3}$51.28{\scriptstyle \pm 1.62}$ & \cellcolor{c1}$55.21{\scriptstyle \pm 2.44}$ & $50.48{\scriptstyle \pm 0.71}$ & \cellcolor{c2}$51.61{\scriptstyle \pm 0.34}$ \\
 & 512 & $2.63{\scriptstyle \pm 0.55}$ & $3.53{\scriptstyle \pm 0.40}$ & $29.41{\scriptstyle \pm 3.29}$ & $1.70{\scriptstyle \pm 0.03}$ & $2.24{\scriptstyle \pm 0.58}$ & $3.24{\scriptstyle \pm 3.74}$ & \cellcolor{c2}$66.59{\scriptstyle \pm 0.45}$ & $64.29{\scriptstyle \pm 0.43}$ & \cellcolor{c1}$67.49{\scriptstyle \pm 0.30}$ & \cellcolor{c3}$64.64{\scriptstyle \pm 0.88}$ \\
\bottomrule
\end{tabular}}
\end{table}

\begin{table}[h]\centering
\caption{Link prediction. Top-3 highlighted per row.}\label{tab:lp}
\resizebox{\columnwidth}{!}{%
\begin{tabular}{llcccccccc}
\toprule
Data & Set & GNN & GraphCL & AnyG & GPF & GP & GFT & \modelname{} & \modelname{}++ \\
\midrule
\multirow{3}{*}{Photo} & AP & $74.53{\scriptstyle \pm 3.77}$ & \cellcolor{c3}$80.16{\scriptstyle \pm 0.70}$ & \cellcolor{c1}$98.05{\scriptstyle \pm 1.11}$ & $78.54{\scriptstyle \pm 0.09}$ & $77.07{\scriptstyle \pm 1.11}$ & $50.26{\scriptstyle \pm 20.95}$ & $66.07{\scriptstyle \pm 22.12}$ & \cellcolor{c2}$88.16{\scriptstyle \pm 1.60}$ \\
 & AUC & $69.72{\scriptstyle \pm 7.07}$ & $76.19{\scriptstyle \pm 1.04}$ & \cellcolor{c1}$97.65{\scriptstyle \pm 1.32}$ & \cellcolor{c3}$77.07{\scriptstyle \pm 0.10}$ & $74.96{\scriptstyle \pm 1.25}$ & $40.69{\scriptstyle \pm 36.11}$ & $62.46{\scriptstyle \pm 29.57}$ & \cellcolor{c2}$88.53{\scriptstyle \pm 1.38}$ \\
 & H@1k & $1.02{\scriptstyle \pm 0.00}$ & \cellcolor{c1}$1.03{\scriptstyle \pm 0.00}$ & \cellcolor{c2}$1.03{\scriptstyle \pm 0.00}$ & \cellcolor{c3}$1.03{\scriptstyle \pm 0.00}$ & $1.02{\scriptstyle \pm 0.00}$ & $0.43{\scriptstyle \pm 0.54}$ & $1.02{\scriptstyle \pm 0.02}$ & $1.01{\scriptstyle \pm 0.01}$ \\
\midrule
\multirow{3}{*}{Goodreads} & AP & $69.22{\scriptstyle \pm 1.19}$ & \cellcolor{c3}$69.23{\scriptstyle \pm 0.79}$ & \cellcolor{c1}$95.62{\scriptstyle \pm 3.30}$ & $65.26{\scriptstyle \pm 0.07}$ & $67.63{\scriptstyle \pm 3.10}$ & $44.52{\scriptstyle \pm 15.41}$ & $59.08{\scriptstyle \pm 14.81}$ & \cellcolor{c2}$80.96{\scriptstyle \pm 0.61}$ \\
 & AUC & $65.73{\scriptstyle \pm 2.31}$ & \cellcolor{c3}$66.58{\scriptstyle \pm 0.86}$ & \cellcolor{c1}$95.50{\scriptstyle \pm 3.31}$ & $63.94{\scriptstyle \pm 0.05}$ & $66.50{\scriptstyle \pm 1.58}$ & $32.87{\scriptstyle \pm 30.10}$ & $57.26{\scriptstyle \pm 15.89}$ & \cellcolor{c2}$84.81{\scriptstyle \pm 0.81}$ \\
 & H@1k & \cellcolor{c1}$0.06{\scriptstyle \pm 0.00}$ & \cellcolor{c2}$0.06{\scriptstyle \pm 0.00}$ & \cellcolor{c3}$0.06{\scriptstyle \pm 0.00}$ & $0.06{\scriptstyle \pm 0.00}$ & $0.06{\scriptstyle \pm 0.00}$ & $0.03{\scriptstyle \pm 0.03}$ & $0.05{\scriptstyle \pm 0.01}$ & $0.06{\scriptstyle \pm 0.00}$ \\
\midrule
\multirow{3}{*}{Fitness} & AP & $76.65{\scriptstyle \pm 3.35}$ & \cellcolor{c3}$83.53{\scriptstyle \pm 0.71}$ & \cellcolor{c1}$97.89{\scriptstyle \pm 0.26}$ & $81.70{\scriptstyle \pm 0.04}$ & $80.75{\scriptstyle \pm 0.81}$ & $49.89{\scriptstyle \pm 20.72}$ & $73.47{\scriptstyle \pm 18.12}$ & \cellcolor{c2}$88.78{\scriptstyle \pm 2.06}$ \\
 & AUC & $71.29{\scriptstyle \pm 7.21}$ & $79.41{\scriptstyle \pm 1.12}$ & \cellcolor{c1}$97.42{\scriptstyle \pm 0.20}$ & \cellcolor{c3}$79.43{\scriptstyle \pm 0.02}$ & $77.94{\scriptstyle \pm 1.69}$ & $41.44{\scriptstyle \pm 36.98}$ & $75.63{\scriptstyle \pm 15.34}$ & \cellcolor{c2}$89.15{\scriptstyle \pm 1.50}$ \\
 & H@1k & \cellcolor{c1}$0.30{\scriptstyle \pm 0.00}$ & \cellcolor{c2}$0.30{\scriptstyle \pm 0.00}$ & \cellcolor{c3}$0.30{\scriptstyle \pm 0.00}$ & $0.30{\scriptstyle \pm 0.00}$ & $0.30{\scriptstyle \pm 0.00}$ & $0.08{\scriptstyle \pm 0.13}$ & $0.27{\scriptstyle \pm 0.04}$ & $0.29{\scriptstyle \pm 0.00}$ \\
\midrule
\multirow{3}{*}{gowalla} & AP & $78.56{\scriptstyle \pm 3.90}$ & $80.09{\scriptstyle \pm 0.18}$ & \cellcolor{c1}$97.64{\scriptstyle \pm 0.34}$ & $78.04{\scriptstyle \pm 0.04}$ & $76.34{\scriptstyle \pm 0.81}$ & $30.69{\scriptstyle \pm 0.00}$ & \cellcolor{c2}$95.74{\scriptstyle \pm 0.17}$ & \cellcolor{c3}$84.22{\scriptstyle \pm 2.44}$ \\
 & AUC & $75.38{\scriptstyle \pm 5.95}$ & $76.83{\scriptstyle \pm 0.30}$ & \cellcolor{c1}$97.36{\scriptstyle \pm 0.33}$ & $74.21{\scriptstyle \pm 0.03}$ & $71.77{\scriptstyle \pm 1.63}$ & $0.02{\scriptstyle \pm 0.03}$ & \cellcolor{c2}$95.30{\scriptstyle \pm 0.23}$ & \cellcolor{c3}$84.27{\scriptstyle \pm 1.97}$ \\
 & H@1k & \cellcolor{c3}$0.23{\scriptstyle \pm 0.00}$ & $0.23{\scriptstyle \pm 0.00}$ & \cellcolor{c1}$0.24{\scriptstyle \pm 0.00}$ & $0.23{\scriptstyle \pm 0.00}$ & $0.23{\scriptstyle \pm 0.00}$ & $0.00{\scriptstyle \pm 0.00}$ & \cellcolor{c2}$0.24{\scriptstyle \pm 0.00}$ & $0.23{\scriptstyle \pm 0.00}$ \\
\midrule
\multirow{3}{*}{Arxiv} & AP & $73.23{\scriptstyle \pm 2.81}$ & \cellcolor{c3}$81.26{\scriptstyle \pm 0.91}$ & \cellcolor{c1}$99.24{\scriptstyle \pm 0.08}$ & $76.57{\scriptstyle \pm 0.07}$ & $77.78{\scriptstyle \pm 0.90}$ & $49.73{\scriptstyle \pm 20.12}$ & $81.03{\scriptstyle \pm 5.88}$ & \cellcolor{c2}$94.34{\scriptstyle \pm 0.29}$ \\
 & AUC & $67.26{\scriptstyle \pm 6.33}$ & $77.03{\scriptstyle \pm 1.22}$ & \cellcolor{c1}$98.98{\scriptstyle \pm 0.10}$ & $73.49{\scriptstyle \pm 0.09}$ & $75.35{\scriptstyle \pm 1.76}$ & $40.28{\scriptstyle \pm 35.63}$ & \cellcolor{c3}$78.25{\scriptstyle \pm 5.94}$ & \cellcolor{c2}$94.54{\scriptstyle \pm 0.49}$ \\
 & H@1k & \cellcolor{c1}$0.38{\scriptstyle \pm 0.00}$ & \cellcolor{c2}$0.38{\scriptstyle \pm 0.00}$ & \cellcolor{c3}$0.38{\scriptstyle \pm 0.00}$ & $0.38{\scriptstyle \pm 0.00}$ & $0.38{\scriptstyle \pm 0.00}$ & $0.13{\scriptstyle \pm 0.17}$ & $0.37{\scriptstyle \pm 0.00}$ & $0.38{\scriptstyle \pm 0.00}$ \\
\midrule
\multirow{3}{*}{Cora} & AP & $60.78{\scriptstyle \pm 1.24}$ & $80.45{\scriptstyle \pm 0.95}$ & \cellcolor{c2}$99.23{\scriptstyle \pm 0.09}$ & $63.31{\scriptstyle \pm 0.05}$ & $69.53{\scriptstyle \pm 4.11}$ & $38.77{\scriptstyle \pm 8.16}$ & \cellcolor{c1}$99.56{\scriptstyle \pm 0.24}$ & \cellcolor{c3}$92.11{\scriptstyle \pm 0.54}$ \\
 & AUC & $61.57{\scriptstyle \pm 4.71}$ & $76.79{\scriptstyle \pm 1.13}$ & \cellcolor{c2}$98.86{\scriptstyle \pm 0.10}$ & $60.20{\scriptstyle \pm 0.50}$ & $68.15{\scriptstyle \pm 6.21}$ & $25.39{\scriptstyle \pm 22.55}$ & \cellcolor{c1}$99.49{\scriptstyle \pm 0.23}$ & \cellcolor{c3}$93.16{\scriptstyle \pm 0.43}$ \\
 & H@1k & $39.12{\scriptstyle \pm 2.63}$ & $52.72{\scriptstyle \pm 1.03}$ & \cellcolor{c1}$62.53{\scriptstyle \pm 0.04}$ & $39.20{\scriptstyle \pm 0.06}$ & $42.14{\scriptstyle \pm 2.66}$ & $10.75{\scriptstyle \pm 10.20}$ & \cellcolor{c2}$62.53{\scriptstyle \pm 0.10}$ & \cellcolor{c3}$58.71{\scriptstyle \pm 0.97}$ \\
\midrule
\multirow{3}{*}{CS} & AP & $58.58{\scriptstyle \pm 3.33}$ & $76.60{\scriptstyle \pm 0.61}$ & \cellcolor{c1}$99.44{\scriptstyle \pm 0.19}$ & $68.39{\scriptstyle \pm 0.06}$ & $69.43{\scriptstyle \pm 1.79}$ & $45.51{\scriptstyle \pm 14.15}$ & \cellcolor{c3}$88.17{\scriptstyle \pm 7.92}$ & \cellcolor{c2}$90.05{\scriptstyle \pm 0.38}$ \\
 & AUC & $58.73{\scriptstyle \pm 0.63}$ & $73.39{\scriptstyle \pm 0.60}$ & \cellcolor{c1}$99.32{\scriptstyle \pm 0.25}$ & $66.00{\scriptstyle \pm 0.15}$ & $67.89{\scriptstyle \pm 2.69}$ & $33.77{\scriptstyle \pm 29.33}$ & \cellcolor{c3}$86.43{\scriptstyle \pm 8.41}$ & \cellcolor{c2}$91.53{\scriptstyle \pm 0.20}$ \\
 & H@1k & $3.70{\scriptstyle \pm 0.76}$ & $4.85{\scriptstyle \pm 0.03}$ & \cellcolor{c1}$4.99{\scriptstyle \pm 0.00}$ & $4.64{\scriptstyle \pm 0.02}$ & $4.62{\scriptstyle \pm 0.05}$ & $1.96{\scriptstyle \pm 2.23}$ & \cellcolor{c2}$4.97{\scriptstyle \pm 0.02}$ & \cellcolor{c3}$4.87{\scriptstyle \pm 0.04}$ \\
\midrule
\multirow{3}{*}{Collab} & AP & $63.65{\scriptstyle \pm 0.80}$ & $80.93{\scriptstyle \pm 0.50}$ & \cellcolor{c1}$99.80{\scriptstyle \pm 0.02}$ & $73.04{\scriptstyle \pm 0.05}$ & $74.18{\scriptstyle \pm 2.06}$ & $46.20{\scriptstyle \pm 15.54}$ & \cellcolor{c2}$99.75{\scriptstyle \pm 0.01}$ & \cellcolor{c3}$96.70{\scriptstyle \pm 0.19}$ \\
 & AUC & $62.39{\scriptstyle \pm 1.98}$ & $76.84{\scriptstyle \pm 0.47}$ & \cellcolor{c1}$99.72{\scriptstyle \pm 0.02}$ & $69.31{\scriptstyle \pm 0.04}$ & $70.78{\scriptstyle \pm 3.35}$ & $33.73{\scriptstyle \pm 29.36}$ & \cellcolor{c2}$99.69{\scriptstyle \pm 0.01}$ & \cellcolor{c3}$97.10{\scriptstyle \pm 0.33}$ \\
 & H@1k & \cellcolor{c1}$0.41{\scriptstyle \pm 0.00}$ & \cellcolor{c2}$0.41{\scriptstyle \pm 0.00}$ & \cellcolor{c3}$0.41{\scriptstyle \pm 0.00}$ & $0.41{\scriptstyle \pm 0.00}$ & $0.41{\scriptstyle \pm 0.00}$ & $0.17{\scriptstyle \pm 0.19}$ & $0.41{\scriptstyle \pm 0.00}$ & $0.40{\scriptstyle \pm 0.00}$ \\
\midrule
\multirow{3}{*}{web-Stan} & AP & $75.73{\scriptstyle \pm 6.11}$ & \cellcolor{c3}$80.38{\scriptstyle \pm 1.68}$ & \cellcolor{c1}$99.60{\scriptstyle \pm 0.08}$ & $74.51{\scriptstyle \pm 0.09}$ & $75.97{\scriptstyle \pm 1.20}$ & $49.96{\scriptstyle \pm 19.69}$ & $69.56{\scriptstyle \pm 15.58}$ & \cellcolor{c2}$85.97{\scriptstyle \pm 0.30}$ \\
 & AUC & $69.11{\scriptstyle \pm 7.13}$ & \cellcolor{c3}$76.17{\scriptstyle \pm 1.74}$ & \cellcolor{c1}$99.53{\scriptstyle \pm 0.08}$ & $70.71{\scriptstyle \pm 0.11}$ & $75.53{\scriptstyle \pm 1.92}$ & $41.91{\scriptstyle \pm 37.25}$ & $65.84{\scriptstyle \pm 17.19}$ & \cellcolor{c2}$85.92{\scriptstyle \pm 0.45}$ \\
 & H@1k & \cellcolor{c1}$0.22{\scriptstyle \pm 0.00}$ & \cellcolor{c2}$0.22{\scriptstyle \pm 0.00}$ & \cellcolor{c3}$0.22{\scriptstyle \pm 0.00}$ & $0.22{\scriptstyle \pm 0.00}$ & $0.22{\scriptstyle \pm 0.00}$ & $0.06{\scriptstyle \pm 0.10}$ & $0.21{\scriptstyle \pm 0.01}$ & $0.22{\scriptstyle \pm 0.00}$ \\
\midrule
\multirow{3}{*}{roadNet} & AP & $57.04{\scriptstyle \pm 4.13}$ & \cellcolor{c3}$84.23{\scriptstyle \pm 1.09}$ & \cellcolor{c1}$99.76{\scriptstyle \pm 0.04}$ & $58.08{\scriptstyle \pm 0.07}$ & $67.05{\scriptstyle \pm 3.72}$ & $38.39{\scriptstyle \pm 11.96}$ & $59.63{\scriptstyle \pm 2.93}$ & \cellcolor{c2}$95.69{\scriptstyle \pm 0.85}$ \\
 & AUC & $59.96{\scriptstyle \pm 6.08}$ & \cellcolor{c3}$81.73{\scriptstyle \pm 0.95}$ & \cellcolor{c1}$99.61{\scriptstyle \pm 0.09}$ & $56.43{\scriptstyle \pm 0.22}$ & $69.04{\scriptstyle \pm 3.83}$ & $20.09{\scriptstyle \pm 26.87}$ & $54.60{\scriptstyle \pm 2.94}$ & \cellcolor{c2}$96.74{\scriptstyle \pm 0.47}$ \\
 & H@1k & $0.12{\scriptstyle \pm 0.05}$ & \cellcolor{c1}$0.19{\scriptstyle \pm 0.00}$ & \cellcolor{c2}$0.19{\scriptstyle \pm 0.00}$ & $0.18{\scriptstyle \pm 0.00}$ & \cellcolor{c3}$0.19{\scriptstyle \pm 0.00}$ & $0.05{\scriptstyle \pm 0.09}$ & $0.19{\scriptstyle \pm 0.00}$ & $0.19{\scriptstyle \pm 0.00}$ \\
\bottomrule
\end{tabular}}
\end{table}

\begin{table}[h]\centering
\caption{Graph classification accuracy (\%). Top-3 highlighted per row.}\label{tab:gc}
\resizebox{\columnwidth}{!}{%
\begin{tabular}{llcccccccccc}
\toprule
Data & Set & GNN & GraphCL & AnyG & GPF & GP & GFT & \modelname{}$_m$ & \modelname{}$_L$ & \modelname{}++$_m$ & \modelname{}++$_L$ \\
\midrule
\multirow{4}{*}{MUTAG} & 4 & \cellcolor{c2}$72.22{\scriptstyle \pm 14.70}$ & $64.81{\scriptstyle \pm 19.51}$ & $61.11{\scriptstyle \pm 19.25}$ & \cellcolor{c1}$79.63{\scriptstyle \pm 3.21}$ & $57.41{\scriptstyle \pm 8.49}$ & \cellcolor{c3}$72.22{\scriptstyle \pm 14.70}$ & $70.37{\scriptstyle \pm 3.21}$ & $66.67{\scriptstyle \pm 9.62}$ & $62.96{\scriptstyle \pm 13.98}$ & $61.11{\scriptstyle \pm 5.56}$ \\
 & 16 & $70.37{\scriptstyle \pm 8.49}$ & \cellcolor{c2}$77.78{\scriptstyle \pm 11.11}$ & $68.52{\scriptstyle \pm 16.04}$ & \cellcolor{c3}$74.07{\scriptstyle \pm 8.49}$ & $57.41{\scriptstyle \pm 8.49}$ & \cellcolor{c1}$79.63{\scriptstyle \pm 3.21}$ & $66.67{\scriptstyle \pm 5.56}$ & $70.37{\scriptstyle \pm 3.21}$ & $55.56{\scriptstyle \pm 0.00}$ & $68.52{\scriptstyle \pm 8.49}$ \\
 & 64 & $70.37{\scriptstyle \pm 3.21}$ & \cellcolor{c2}$77.78{\scriptstyle \pm 11.11}$ & $64.81{\scriptstyle \pm 11.56}$ & \cellcolor{c1}$79.63{\scriptstyle \pm 3.21}$ & $62.96{\scriptstyle \pm 11.56}$ & $68.52{\scriptstyle \pm 11.56}$ & \cellcolor{c3}$75.93{\scriptstyle \pm 6.42}$ & $72.22{\scriptstyle \pm 5.56}$ & $62.96{\scriptstyle \pm 3.21}$ & $74.07{\scriptstyle \pm 13.98}$ \\
 & 512 & -- & -- & -- & -- & -- & -- & -- & -- & -- & -- \\
\midrule
\multirow{4}{*}{PROT} & 4 & \cellcolor{c2}$53.15{\scriptstyle \pm 13.45}$ & $48.35{\scriptstyle \pm 8.07}$ & \cellcolor{c1}$62.76{\scriptstyle \pm 5.43}$ & $47.45{\scriptstyle \pm 3.16}$ & $50.45{\scriptstyle \pm 4.68}$ & $52.85{\scriptstyle \pm 10.05}$ & $48.95{\scriptstyle \pm 3.41}$ & $51.65{\scriptstyle \pm 3.41}$ & $51.35{\scriptstyle \pm 5.02}$ & \cellcolor{c3}$53.15{\scriptstyle \pm 3.25}$ \\
 & 16 & $41.14{\scriptstyle \pm 5.20}$ & $57.06{\scriptstyle \pm 6.33}$ & \cellcolor{c2}$63.96{\scriptstyle \pm 8.87}$ & $52.25{\scriptstyle \pm 6.80}$ & $49.85{\scriptstyle \pm 4.96}$ & $50.45{\scriptstyle \pm 14.30}$ & \cellcolor{c3}$62.76{\scriptstyle \pm 5.72}$ & \cellcolor{c1}$65.17{\scriptstyle \pm 5.43}$ & $53.45{\scriptstyle \pm 0.52}$ & $57.06{\scriptstyle \pm 4.96}$ \\
 & 64 & $54.35{\scriptstyle \pm 9.63}$ & $58.86{\scriptstyle \pm 5.79}$ & \cellcolor{c1}$65.47{\scriptstyle \pm 6.82}$ & $57.06{\scriptstyle \pm 5.72}$ & $54.65{\scriptstyle \pm 5.50}$ & $48.35{\scriptstyle \pm 14.31}$ & $61.56{\scriptstyle \pm 3.64}$ & \cellcolor{c2}$64.56{\scriptstyle \pm 2.90}$ & \cellcolor{c3}$64.26{\scriptstyle \pm 4.53}$ & $62.46{\scriptstyle \pm 3.41}$ \\
 & 512 & $46.85{\scriptstyle \pm 8.01}$ & $58.86{\scriptstyle \pm 5.28}$ & \cellcolor{c1}$66.97{\scriptstyle \pm 7.50}$ & $56.46{\scriptstyle \pm 5.50}$ & $53.45{\scriptstyle \pm 1.38}$ & $48.95{\scriptstyle \pm 8.17}$ & $60.36{\scriptstyle \pm 3.25}$ & \cellcolor{c3}$64.26{\scriptstyle \pm 1.88}$ & \cellcolor{c2}$66.07{\scriptstyle \pm 3.16}$ & $63.06{\scriptstyle \pm 7.85}$ \\
\midrule
\multirow{4}{*}{NCI1} & 4 & $51.09{\scriptstyle \pm 4.94}$ & \cellcolor{c1}$54.18{\scriptstyle \pm 8.65}$ & \cellcolor{c3}$52.72{\scriptstyle \pm 9.20}$ & $51.58{\scriptstyle \pm 2.56}$ & $51.91{\scriptstyle \pm 0.56}$ & \cellcolor{c2}$53.53{\scriptstyle \pm 11.61}$ & $50.20{\scriptstyle \pm 4.61}$ & $49.72{\scriptstyle \pm 3.55}$ & $51.74{\scriptstyle \pm 2.46}$ & $51.74{\scriptstyle \pm 2.50}$ \\
 & 16 & $55.39{\scriptstyle \pm 2.03}$ & \cellcolor{c2}$57.50{\scriptstyle \pm 2.85}$ & \cellcolor{c3}$55.96{\scriptstyle \pm 7.60}$ & $54.99{\scriptstyle \pm 1.69}$ & $50.69{\scriptstyle \pm 3.96}$ & \cellcolor{c1}$59.29{\scriptstyle \pm 1.34}$ & $48.34{\scriptstyle \pm 4.37}$ & $49.96{\scriptstyle \pm 2.58}$ & $55.23{\scriptstyle \pm 4.28}$ & $55.80{\scriptstyle \pm 5.55}$ \\
 & 64 & $54.58{\scriptstyle \pm 5.96}$ & \cellcolor{c1}$60.75{\scriptstyle \pm 2.90}$ & \cellcolor{c2}$59.04{\scriptstyle \pm 4.04}$ & \cellcolor{c3}$56.20{\scriptstyle \pm 2.76}$ & $51.09{\scriptstyle \pm 1.84}$ & $55.31{\scriptstyle \pm 9.44}$ & $53.28{\scriptstyle \pm 2.35}$ & $55.80{\scriptstyle \pm 2.04}$ & $54.10{\scriptstyle \pm 2.85}$ & $53.85{\scriptstyle \pm 5.47}$ \\
 & 512 & $56.20{\scriptstyle \pm 6.05}$ & \cellcolor{c1}$62.53{\scriptstyle \pm 3.16}$ & \cellcolor{c2}$60.18{\scriptstyle \pm 2.21}$ & $55.88{\scriptstyle \pm 3.24}$ & $51.74{\scriptstyle \pm 1.49}$ & \cellcolor{c3}$59.61{\scriptstyle \pm 2.56}$ & $58.15{\scriptstyle \pm 2.32}$ & $56.53{\scriptstyle \pm 1.97}$ & $54.50{\scriptstyle \pm 4.24}$ & $54.42{\scriptstyle \pm 1.79}$ \\
\midrule
\multirow{4}{*}{IMDB-B} & 4 & \cellcolor{c1}$67.67{\scriptstyle \pm 4.62}$ & $52.00{\scriptstyle \pm 6.08}$ & $51.33{\scriptstyle \pm 5.03}$ & \cellcolor{c2}$61.00{\scriptstyle \pm 2.65}$ & $59.67{\scriptstyle \pm 11.02}$ & \cellcolor{c3}$61.00{\scriptstyle \pm 8.89}$ & $47.00{\scriptstyle \pm 7.00}$ & $52.33{\scriptstyle \pm 14.84}$ & $51.33{\scriptstyle \pm 12.01}$ & $50.67{\scriptstyle \pm 8.50}$ \\
 & 16 & \cellcolor{c1}$63.67{\scriptstyle \pm 5.51}$ & $55.67{\scriptstyle \pm 5.03}$ & $50.00{\scriptstyle \pm 3.46}$ & \cellcolor{c2}$63.67{\scriptstyle \pm 6.51}$ & \cellcolor{c3}$63.00{\scriptstyle \pm 6.00}$ & $60.00{\scriptstyle \pm 6.56}$ & $53.33{\scriptstyle \pm 5.77}$ & $52.33{\scriptstyle \pm 7.09}$ & $55.00{\scriptstyle \pm 4.58}$ & $55.67{\scriptstyle \pm 2.08}$ \\
 & 64 & \cellcolor{c3}$59.00{\scriptstyle \pm 15.87}$ & $55.67{\scriptstyle \pm 6.66}$ & $46.33{\scriptstyle \pm 3.21}$ & \cellcolor{c1}$60.67{\scriptstyle \pm 15.28}$ & $55.00{\scriptstyle \pm 12.53}$ & $56.67{\scriptstyle \pm 9.07}$ & $56.00{\scriptstyle \pm 3.00}$ & $55.33{\scriptstyle \pm 4.51}$ & \cellcolor{c2}$59.33{\scriptstyle \pm 1.15}$ & $57.67{\scriptstyle \pm 4.62}$ \\
 & 512 & \cellcolor{c1}$68.67{\scriptstyle \pm 3.21}$ & $55.67{\scriptstyle \pm 9.07}$ & $53.00{\scriptstyle \pm 2.65}$ & $52.33{\scriptstyle \pm 19.73}$ & $61.33{\scriptstyle \pm 4.16}$ & $60.67{\scriptstyle \pm 5.69}$ & $59.33{\scriptstyle \pm 5.51}$ & \cellcolor{c3}$61.67{\scriptstyle \pm 5.03}$ & $57.67{\scriptstyle \pm 4.51}$ & \cellcolor{c2}$64.67{\scriptstyle \pm 4.51}$ \\
\midrule
\multirow{4}{*}{COLLAB} & 4 & \cellcolor{c1}$52.00{\scriptstyle \pm 12.44}$ & $38.20{\scriptstyle \pm 6.48}$ & $42.07{\scriptstyle \pm 5.77}$ & $45.20{\scriptstyle \pm 4.39}$ & $39.20{\scriptstyle \pm 4.61}$ & $42.67{\scriptstyle \pm 15.17}$ & $49.40{\scriptstyle \pm 4.61}$ & \cellcolor{c2}$50.33{\scriptstyle \pm 7.62}$ & $49.00{\scriptstyle \pm 6.60}$ & \cellcolor{c3}$50.20{\scriptstyle \pm 5.54}$ \\
 & 16 & \cellcolor{c1}$48.80{\scriptstyle \pm 19.92}$ & $33.33{\scriptstyle \pm 0.42}$ & $40.40{\scriptstyle \pm 6.06}$ & $39.27{\scriptstyle \pm 2.61}$ & $34.53{\scriptstyle \pm 0.83}$ & $35.73{\scriptstyle \pm 9.83}$ & $46.40{\scriptstyle \pm 11.58}$ & \cellcolor{c3}$47.27{\scriptstyle \pm 4.31}$ & \cellcolor{c2}$47.40{\scriptstyle \pm 7.29}$ & $47.07{\scriptstyle \pm 7.57}$ \\
 & 64 & $49.00{\scriptstyle \pm 13.71}$ & $32.67{\scriptstyle \pm 1.68}$ & $34.80{\scriptstyle \pm 3.29}$ & $39.67{\scriptstyle \pm 4.29}$ & $38.87{\scriptstyle \pm 5.33}$ & $40.20{\scriptstyle \pm 5.00}$ & \cellcolor{c2}$52.27{\scriptstyle \pm 4.82}$ & $48.00{\scriptstyle \pm 6.64}$ & \cellcolor{c1}$53.40{\scriptstyle \pm 5.13}$ & \cellcolor{c3}$51.60{\scriptstyle \pm 6.07}$ \\
 & 512 & $50.00{\scriptstyle \pm 14.11}$ & $30.67{\scriptstyle \pm 1.10}$ & $35.53{\scriptstyle \pm 2.70}$ & $35.53{\scriptstyle \pm 1.62}$ & $36.93{\scriptstyle \pm 2.87}$ & $38.47{\scriptstyle \pm 5.52}$ & \cellcolor{c3}$58.33{\scriptstyle \pm 2.32}$ & $55.33{\scriptstyle \pm 1.60}$ & \cellcolor{c2}$59.80{\scriptstyle \pm 4.16}$ & \cellcolor{c1}$62.00{\scriptstyle \pm 2.31}$ \\
\bottomrule
\end{tabular}}
\end{table}


\end{document}